\documentclass[12pt]{article}

\usepackage[utf8]{inputenc} 
\usepackage[T1]{fontenc}    
\usepackage{url}            
\usepackage{booktabs}       
\usepackage{amsfonts}       
\usepackage{nicefrac}       
\usepackage{microtype}      
\usepackage{xcolor}         

\usepackage{color}
\usepackage{multirow}
\usepackage{graphicx}
\usepackage{amssymb,amsmath} 
\usepackage{amsthm}
\usepackage{amsbsy}
\usepackage{mathtools}
\usepackage{enumitem}
\usepackage{fullpage}

\usepackage{thmtools,thm-restate}

\usepackage{algorithm}
\usepackage{algpseudocode}
\usepackage{tikz}
\usetikzlibrary{fit,calc,tikzmark}
\makeatletter
\algrenewcommand\algorithmicrequire{\textbf{Input:}}
\algrenewcommand\alglinenumber[1]{\textcolor{gray!70}{\footnotesize#1}}
\algrenewcommand\algorithmiccomment[1]{\hfill\textcolor{gray!70}{\scriptsize$\triangleright$~#1}}
\makeatother

\newlength{\AlgBoxPadLeft}   
\newlength{\AlgBoxPadRight}  
\newlength{\AlgBoxPadTop}    
\newlength{\AlgBoxPadBottom} 

\newcommand{\nuc}{\operatorname{nuc}}

\DeclareMathOperator*{\polar}{polar}

\definecolor{shade}{rgb}{0.9,0.9,0.9}
\usepackage{bm}

\newcommand{\R}{\mathbb{R}}

\newcommand{\cC}{\mathcal{C}}

\newcommand{\op}{\operatorname{op}}

\newcommand{\dotprod}[1]{\left< #1\right>} 
 
\DeclareMathOperator{\argmininn}{argmin} 
\DeclareMathOperator{\argmaxinn}{argmax} 
\newcommand{\argmin}[1]{ \underset{#1}{\argmininn} \;}
\newcommand{\argmax}[1]{ \underset{#1}{\argmaxinn}\;}
\DeclareMathOperator{\traceOp}{tr}  
\DeclareMathOperator{\sign}{sign}   
\DeclareMathOperator{\Sign}{SIGN}   
\DeclareMathOperator{\diag}{diag}  
\DeclareMathOperator{\Diag}{Diag}  
\providecommand{\trace}[1]{\traceOp\left(#1\right)}

\newcommand{\LMO}{{\sf LMO}}

\usepackage{mdframed} 

\definecolor{shadecolor}{gray}{0.90}
\declaretheoremstyle[
headfont=\normalfont\bfseries,
notefont=\mdseries, notebraces={(}{)},
bodyfont=\normalfont,
postheadspace=0.5em,
spaceabove=6pt,
mdframed={
  skipabove=8pt,
  skipbelow=8pt,
  hidealllines=true,
  backgroundcolor={shadecolor},
  innerleftmargin=4pt,
  innerrightmargin=4pt} 
]{shaded}

\declaretheorem[style=shaded,within=section]{definition}
\declaretheorem[style=shaded,sibling=definition]{theorem}
\declaretheorem[style=shaded,sibling=definition]{proposition}

\declaretheorem[style=shaded,sibling=definition]{lemma}
\declaretheorem[style=shaded,sibling=definition]{remark}

\usepackage[most]{tcolorbox}  
\definecolor{myframecolor}{RGB}{0,60,125}
\colorlet{myframecolorlight}{myframecolor!20!white}
\tcbset{
  myremark/.style={
    colback=blue!5!white,
    colframe=myframecolorlight,
    coltitle=black,
    fonttitle=\bfseries,
    boxrule=0.8pt,
    arc=4pt,
    title=#1
  }
}

\usepackage[colorinlistoftodos]{todonotes}

\usepackage[hidelinks]{hyperref}
\usepackage[nameinlink]{cleveref}
\usepackage[
    maxbibnames=99,
  giveninits=false,
  backref=true,
  backend=biber,
  style=alphabetic,
  citestyle=alphabetic,
  url=false,
  isbn=false,
  doi=false
]{biblatex}
\hypersetup{
  colorlinks = true,    
  urlcolor   = blue,    
  linkcolor  = blue,    
  citecolor  = red,     
  filecolor  = black
}
\usepackage{fancyvrb}
\usepackage{array}

\title{Muon Does Not Converge on Convex Lipschitz Functions}

\author{
Tetiana Parshakova\thanks{Center for Computational Mathematics, Flatiron Institute, Simons Foundation}
\and Ahmed Khaled\thanks{Princeton University, Google DeepMind}
\and Michael Crawshaw\thanks{George Mason University}
\and Guillaume Garrigos\thanks{Université Paris Cité and Sorbonne Université, CNRS, Laboratoire de Probabilités, Statistique et Modélisation}
\and Robert M. Gower\footnotemark[1]
}

\date{}

\begin{document}

\maketitle

\begin{abstract}
Muon and its variants have shown strong empirical performance in a variety of deep learning tasks. Existing convergence analyses of Muon rely on smoothness assumptions, though arguably the most successful function class for developing deep learning methods (such as AdaGrad, Shampoo, Schedule-Free and more) has been the class of convex and Lipschitz functions. In this paper we question whether the classical convex Lipschitz model is a useful one for understanding Muon.

Our answer is no. We show that Muon does not converge on the class of convex and Lipschitz functions, regardless of the choice of learning rate schedule. We also show that error feedback restores convergence of Muon and all the non-Euclidean subgradient methods with momentum. 
However, this theoretical fix using error feedback degrades the performance of Muon in 
  two representative settings for image classification (CIFAR-10) and language modeling (nanoGPT on FineWeb-Edu 10B). 
Our conclusion is that convex Lipschitz theory, despite having a prominent role in the design of practical methods for deep learning, is not the most suited one for Muon.
This suggests that Muon's success must come from structure absent from this model, most plausibly related to smoothness.
\end{abstract}

\section{Introduction}

The Muon optimizer~\cite{muon,anthology}, together with related variants~\cite{polargrad,scion}, has recently emerged as a highly competitive optimization approach in a growing range of applications. These include training nanoGPT models at the 124M scale~\cite{modded_nanogpt_2024}, the nanochat challenge at 1.5B parameters~\cite{nanochat_challenge}, and  diffusion models~\cite{schaipp2025optimizationbenchmarkdiffusionmodels}. Muon has also been used in large-scale language model training, for example in the Moonlight 16B (3B active) MoE model~\cite{moonlight}, the 1T (32B active) parameters 
 Kimi frontier model~\cite{team2025kimi} and very recently the DeepSeek-V4 model~\cite{deepseekai2026deepseekv4} with 1.6T (49B active) parameters. These empirical results have generated substantial interest in understanding the theoretical properties of Muon and related spectral methods.

Muon combines spectral descent~\cite{carlson2015stochastic,carlson2015preconditioned} with momentum and a GPU-efficient approximation of the polar factor~\cite{amsel2025polar,muon}. Existing convergence analyses of Muon and related methods predominantly assume some form of smoothness; see \Cref{sec:prior-work} for a detailed review.

By contrast, one of the most influential theoretical settings for the design and analysis of optimization methods in deep learning has been the class of convex Lipschitz objectives. 
Adaptive methods such as AdaGrad~\cite{adagrad} were designed for convex and Lipschitz functions, which later gave rise to  RMSProp and Adam~\cite{adam}. It also extends naturally to matrix-valued methods such as Shampoo~\cite{pmlr-v80-gupta18a}, which can be viewed as a layerwise Kronecker approximation of full-matrix AdaGrad while still enjoying convergence guarantees in the  convex Lipschitz setting. More recently, Schaipp et al.~\cite{pmlr-v267-schaipp25a} showed that convex Lipschitz analyses can be surprisingly predictive of loss dynamics in modern deep learning practice. 
In the inaugural AlgoPerf competition~\cite{Dahl2023AlgoPerf}, the winning methods in the two rulesets were distributed Shampoo and Schedule-Free AdamW~\cite{Kasimbeg2025AlgoPerfResults}, both of which are closely tied to convex Lipschitz theory. In particular, Schedule-Free is motivated by optimal last-iterate convergence guarantees in the convex Lipschitz setting. Yet, despite the close conceptual connection between spectral descent and matrix-preconditioned methods, there is currently no convergence theory for Muon or spectral descent in the convex Lipschitz setting.

A central reason for this gap is that Muon may fail to converge on this class. 
We construct a convex Lipschitz objective, adapted from~\cite{pmlr-v97-karimireddy19a}, for which the iterates of Muon remain bounded away from the solution.
This lack of convergence holds even with a vanishing sequence of stepsizes, but requires an initialization in a narrow set.
We then strengthen this construction and show that even local convergence can fail for almost every initialization, provided the momentum parameter $\beta$ is chosen in $[0,1/2)$.
This stands in sharp contrast both to the non-Euclidean smooth setting~\cite[Theorem 4.6]{shen2025convergenceanalysisofmuon} and to the standard Euclidean nonsmooth setting~\cite[Theorem 9.9]{GarGow23}, for which anytime convergence rates are known when $\beta = 0$.

Our analysis relies on the observation that, on diagonal matrices, Muon is  equivalent to the signed momentum method. We prove this reduction in~\Cref{lem:muon-reduction}. This reduction is especially notable because signed momentum has also been studied as a simple proxy for Adam~\cite{balles2018dissecting,orvieto2025in}. Thus for diagonal matrices, Muon and Adam are in fact closely related methods. With this reduction, signed momentum becomes an exact specialization of Muon and therefore provides a tractable route to analysis. Building on this observation, we extend the counterexample of~\cite{pmlr-v97-karimireddy19a} for signed subgradient descent to prove non-convergence of signed momentum, and hence of Muon.

We then show that there is a modification of Muon that  converges in the convex Lipschitz setting.
We propose EF-Muon, a combination of spectral descent with error feedback~\cite{pmlr-v97-karimireddy19a}, and provide an anytime convergence rate guarantee for this method.
This EF modification can also be combined with any non-Euclidean descent method, resulting in a convergent method.
Empirically, however, this fix does not appear to help in the regimes where Muon is typically used.
Experiments on image classification (WideResNet-28-10 with 36M parameters on CIFAR-10~\cite{krizhevsky2009learning})
and language modeling (a nanoGPT-like transformer with 206M parameters on FineWeb-Edu~\cite{penedo2024fineweb}) tasks 
show that EF-Muon does not improve over standard Muon for practical problems, which are outside the convex Lipschitz class.

Taken together, our results suggest that the convex Lipschitz class might not be the right theoretical lens for understanding Muon. 
In this regime, much like Adam~\cite{ReddiKK18}, Muon itself can fail to converge, and modifying it to recover convergence does not yield a better practical method. This is striking because closely related optimization ideas, from AdaGrad and Shampoo to Schedule-Free, are all meaningfully informed by convex Lipschitz theory. For Muon, by contrast, the standard convex Lipschitz framework does not yield a straightforward positive theory for unmodified Muon, unlike for AdaGrad/Shampoo-style methods. Any convex Lipschitz theory for Muon must either modify the algorithm, restrict the problem class, or exploit assumptions beyond worst-case Lipschitz convexity.

\vspace{-0.1em}

\subsection{Contributions}
Our contributions can be summarized as follows. \\[0.25cm]
\noindent{\bf Reduction of Muon to signed momentum.} We exhibit a reduction of Muon to signed momentum in~\Cref{thm-muon-reduction-informal} for all functions that rely only on the diagonal of a matrix. \\[0.25cm]
\noindent{\bf Counterexamples for Muon.} 
We give Counterexamples~\ref{thm-counterex-dec-steps} and~\ref{Cex:offline stepsizes momentum} that show that Muon does not converge on the class of convex and Lipschitz functions.\\[0.25cm]
\noindent{\bf Non-Euclidean updates are compressions.} We show that the Muon update (specifically: regularized Muon) can be viewed as a compression operator, and that the update of all regularized non-Euclidean subgradient methods are compression operators~(\Cref{L:LMO gives compressor operator}). In the language of Nesterov~\cite{Nes12a}, this means that all sharp operators are, up to a multiplicative constant,  \emph{compression} operators \cite{BezHorRicSaf23}.\\[0.25cm]
\noindent{\bf Error feedback with momentum.} Using this interpretation as a compression, we develop the EF-M (error feedback with momentum) algorithm, which when combined with any non-Euclidean subgradient method results in a method that converges for convex and Lipschitz functions~(\Cref{thm:efm-convex}).

\section{Non-Euclidean subgradient methods}\label{sec-non-eclid-subgrad}

To simplify the discussion, we will consider a loss over a single linear layer. That is, let $W \in \R^{m\times n}$ be a matrix of parameters, and suppose our loss $f(W)\in \R$ is a function of this matrix, and our objective is to find $W$ that minimizes $f$.
Extensions to the setting such as a Cartesian product of matrices, which is the parameter space for deep learning, can be done by using an appropriate max norm over the Cartesian product space, see~\Cref{sec:Cartesian-compressor} and~\cite{anthology,Crawshaw2025}. 

 To define the local linearization of $f(W)$ we use the trace inner-product given by
\[
\dotprod{W, G} = \trace{W^\top G} = \sum_{ij} W_{ij} G_{ij}.
\]
This trace inner-product is equivalent to flattening the matrices $W$ and  $G$ into vectors, 
and then using the standard inner-product over vectors. 

\subsection{Two flavors of subgradient descent}

Let $W_t\in \R^{m\times n}$ be a given iterate at the $t$-th step. For a given norm $\|\cdot\|$ over matrices, we will consider  two variants of subgradient descent for minimizing $f(W)$. \emph{Regularized subgradient descent} is given by minimizing the local linearization with an additive regularizer
\begin{eqnarray}
         W_{t+1} &\in & \argmin{W \in \R^{m\times n} } \left ( f(W_t) + \dotprod{G_t,  W-W_t} +\frac{1}{2\lambda_t} \|W-W_t\|^2 \right ),   \label{eq:reg-gd}
\end{eqnarray}
whereas  \emph{constrained subgradient descent}  is given by minimizing the local linearization subject to a constraint
\begin{eqnarray}
         W_{t+1} &\in & \argmin{W \in \R^{m\times n} } \left (f(W_t) + \dotprod{G_t,  W-W_t}\right ) \quad \mbox{subject to} \quad \|W-W_t\|  \leq \lambda_t \label{eq:const-gd},
\end{eqnarray}
where $\lambda_t>0$ is the learning rate, and $G_t \in \partial f(W_t)$. 
Note that~\eqref{eq:reg-gd} and~\eqref{eq:const-gd} may have more than one solution, so to specify an algorithm one has to choose a particular solution $W_{t+1}$.

The question is, which norm over matrices should we use? 
In neural networks, matrix parameters $W$ represent linear layers that apply the transformation $x \mapsto Wx$, where $x$ denotes the current embedding or input vector. When stacking multiple linear layers with nonlinear activations, stable training often requires controlling the scale of the activations. In particular, the output $x' = Wx$ should remain bounded  so that $\|x'\|_2 = O(1)$ \cite{yang2024spectralconditionfeaturelearning}. One way to enforce this is to control both the weight matrix and the input embedding using
\[
\|Wx\|_2 \le \|W\|_{\op}\|x\|_2,
\]
where $\|W\|_{\op}$ denotes the operator (spectral) norm of $W$, defined as
\begin{equation}
\|W\|_{\op}
=
\sup_{\|x\|_2 =1}  \|Wx\|_2.
\label{eq:spectralnorm}
\end{equation} 
The norm of input $\|x\|_2$ is often bounded due to  some normalization layer, such as Layernorm~\cite{ba2016layernormalization} or RMSnorm~\cite{zhang2019root}.   
To incentivize the operator norm to remain bounded, we can use it to regularize~\eqref{eq:reg-gd}, or constrain~\eqref{eq:const-gd}, the update in subgradient descent. This leads to spectral subgradient descent.

\subsection{Muon and spectral subgradient descent}

Given a matrix $W$ we define its \emph{polar factor} as $\polar(W) := U V^\top$, where $W = U \Sigma V^\top$ is the reduced SVD decomposition of $W$.
If we use the operator norm $\|\cdot\| = \|\cdot\|_{\operatorname{op}}$ in constrained subgradient descent~\eqref{eq:const-gd}, then one of the solutions~\cite{scion} is given by 
\begin{align}
         W_{t+1} &=   W_t - \lambda_t  \polar(G_t),\label{eq:muon-const} \tag{specGD}
\end{align}
where $\lambda_t$ is the learning rate.
The above method is known as \emph{spectral subgradient descent}, and can be combined with momentum to give rise to the \emph{Muon} algorithm\footnote{We note that sometimes Muon uses a Nesterov style momentum~\cite{muon}, though for simplicity we consider only the Polyak style momentum.}:
\begin{align}
M_t &= \beta M_{t-1} + (1-\beta) G_t, 
\quad G_t \in \partial f(W_t),
\quad M_{-1}  =0,
\label{eq:muon} \tag{Muon}
 \\
W_{t+1} &= W_t - \lambda_t \polar(M_t), \nonumber
\end{align}
where $\beta \in [0, 1)$ is the momentum coefficient.
If we consider the regularized subgradient descent~\eqref{eq:reg-gd} instead of~\eqref{eq:const-gd}, then one obtains a similar algorithm~\cite{anthology,Crawshaw2025} dubbed \emph{regularized Muon} where the descent direction $\polar(M_t)$ is rescaled with the nuclear norm $\Vert M_t \Vert_{\nuc}$:
\begin{align}
M_t &= \beta M_{t-1} + (1-\beta) G_t, 
\quad G_t \in \partial f(W_t),
\quad M_{-1}  =0,
\label{eq:muon-reg} \tag{regMuon}
 \\
W_{t+1} &= W_t - \lambda_t \Vert M_t \Vert_{\nuc} \polar(M_t). \nonumber
\end{align}
Note that computing the polar factor in \eqref{eq:muon} can be done efficiently on GPUs without needing to compute the SVD.
Indeed the polar factor can be approximated by a matrix polynomial~\cite{muon,amsel2025polar}, which is how\footnote{As accessed on May 1, 2026.} Muon is  implemented in practice~\cite{modded_nanogpt_2024,nanochat_challenge}.

\section{Counterexamples for Muon on nonsmooth convex problems} \label{sec-counterex-muon}

We are going to show that Muon cannot converge in the convex nonsmooth setting. 
More precisely, that it is not possible to derive an anytime convergence result for the Muon algorithm over the class of convex Lipschitz functions.
The proofs of our claims can be found in \Cref{S:proofs counterexamples}.
Our main tool is the following convex Lipschitz function, which will serve as our main counterexample throughout.

\begin{definition}\label{def:counterexample matrix kinky}
Fix $m,n \geq 2$ and $c\in(0,1)$.
Define $f:\mathbb{R}^{m\times n}\to\mathbb{R}$ by
\begin{equation} \label{eq:kinky-function}
f(W) = c |W_{11}+W_{22}|+|W_{11}-W_{22}|.
\end{equation}
\end{definition}

This is a diagonal function: $f(W)$ depends only on the (first two) diagonal entries of $W$, and as such can naturally be identified with a function over $\mathbb{R}^2$.
In this setting, the Muon update simplifies considerably.
Indeed the polar factor of its subgradients $\polar(G_t)$ in \eqref{eq:muon} is obtained by taking the elementwise sign of $G_t$.
Consequently, Muon reduces exactly to signed momentum on this class of problems, as stated below; see \Cref{S:reduction muon matrix to vector} for the proofs.

\begin{proposition}[Informal]\label{thm-muon-reduction-informal}
If $f:\R^{m \times n} \to \R$ depends only on the main diagonal entries of its argument, 
then Muon 
is equivalent to the signed momentum algorithm.
\end{proposition}

As a consequence, for losses $f$ that only depend on the diagonal of $W$, we can use signed momentum to study the convergence of Muon. 
In particular, if we want to establish convergence of Muon on any class of functions, we should first study signed momentum since it is simpler.

In \cite{pmlr-v97-karimireddy19a} the authors study signed subgradient descent applied to the function $\hat f : \mathbb{R}^2 \to \mathbb{R}$, $\hat f(x,y) = c \vert x + y \vert + \vert x-y \vert$ which inspired our \Cref{def:counterexample matrix kinky}.
They show in \cite[Counterexample~2]{pmlr-v97-karimireddy19a} that signed subgradient descent  without momentum ($\beta=0$) may fail to converge, regardless of how the stepsize schedule is tuned.
Keeping in mind that spectral subgradient descent applied to our function $f$ is equivalent to signed subgradient descent on its diagonal elements, we can immediately deduce that Muon with no momentum ($\beta=0$) may fail to converge too.
But their counterexample suffers from two drawbacks, which we will improve upon.

First, their counterexample does not permit the use of \emph{momentum}, which is standard in the implementation of Muon.
We will show that the obstruction to convergence remains \emph{for every choice of momentum parameter} $\beta \in [0,1)$.

Second, their counterexample 
relies on using a subgradient that is not used in practice.
 To be precise, for the 
subgradient of the absolute value $\vert \cdot \vert$ evaluated at $0$ they use $+1$ or $-1$. Though this is a valid subgradient, 
it goes against standard implementations in PyTorch or TensorFlow, in which the returned subgradient is simply $0$.
We get rid of this technical assumption by showing that Muon will fail to converge \emph{no matter which subgradient is chosen}.

Our first Counterexample~\ref{thm-counterex-dec-steps} below shows that Muon cannot converge globally on the class of convex and Lipschitz functions when using an \emph{offline} decreasing stepsize schedule. 
See \Cref{fig-counterex-muon-trajectory} for an illustration of how the algorithm gets stuck, and see \Cref{S:proofs counterexamples} for the proofs.

\begin{restatable}[Muon cannot converge: offline stepsizes] {counterexample}{maincounterexample}\label{thm-counterex-dec-steps}  
Let $\beta \in [0, 1)$ and let $\{\lambda_t\}_{t=0}^\infty$ be a nonincreasing positive sequence chosen offline.
Let $f$ be as in \Cref{def:counterexample matrix kinky} with $c = \tfrac{1-\beta}{2}$, and consider the iterates $W_t$ of~\eqref{eq:muon} applied to $f$. 
Let $W^\star$ be any minimizer of $f$.

\hspace{1em} $\bullet$ 
If $\lim_{t \rightarrow \infty} \lambda_t > 0$, then there exists an open set of initializations such that  
\begin{equation} \label{eq:nonconv}
    \text{ for every $t \geq 0$,} \quad
    f(W_t)-\inf f \ge 1 - \beta 
    \quad \text{ and } \quad 
    \Vert W_t - W^\star \Vert_F \geq \frac{1 - \beta}{2}.
\end{equation}

\vspace{-0.5em} \hspace{1em} $\bullet$ 
If $\lambda_t$ is strictly decreasing to $0$, then there exists an initialization such that  \eqref{eq:nonconv} holds.
\end{restatable}

\begin{figure}[h!] 
\begin{center}
    \centering 
    \includegraphics[width=\textwidth]{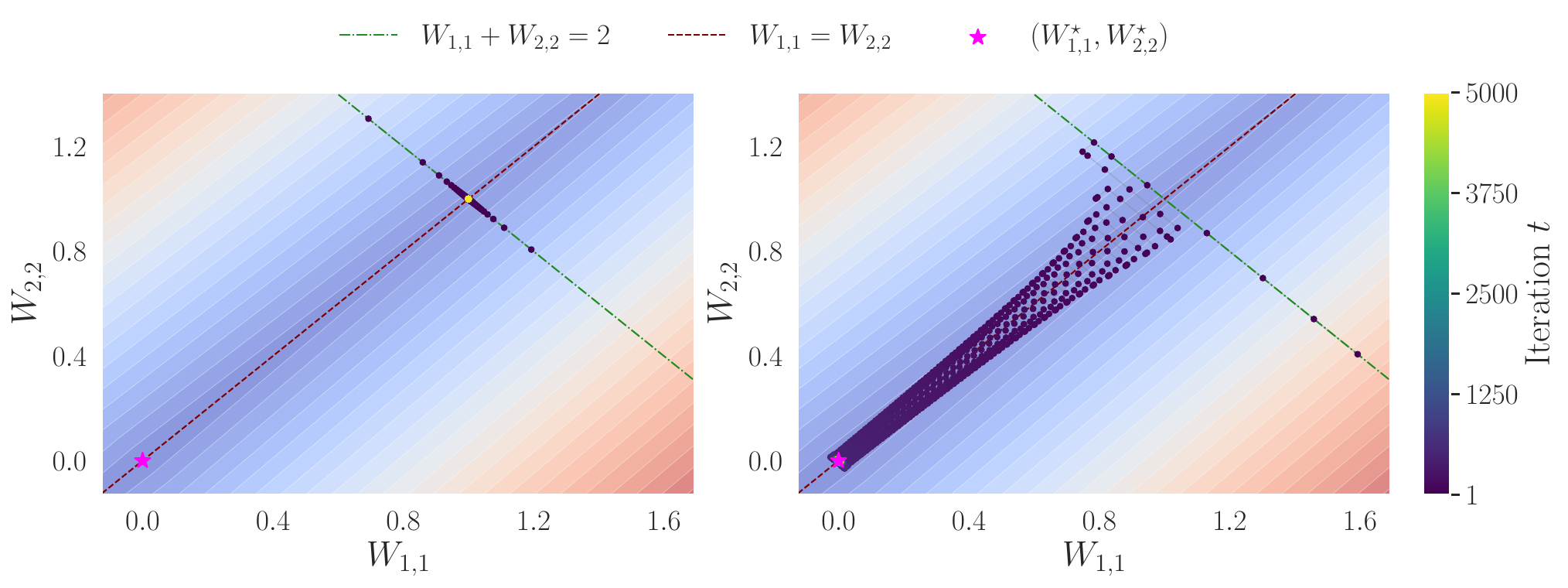}
    \caption{
    Level sets of $f$ given in \Cref{def:counterexample matrix kinky} in the $((W_t)_{1,1},(W_t)_{2,2})$ plane. \\
    Left: Muon iterates. 
    Right: EF-Muon iterates.
    } 
    \label{fig-counterex-muon-trajectory}
\end{center}
\end{figure}

We then relax our assumptions on the stepsizes in our second Counterexample \ref{Cex:offline stepsizes momentum}, in exchange for a restricted use of momentum:
$\beta \in [0, \tfrac{1}{2})$.
More precisely we make no monotonicity assumption on the stepsizes, allowing as well for adaptive schedules depending on the past subgradients.
This includes for instance stepsizes depending on $\|M_t\|_{\nuc}$, such as in \eqref{eq:muon-reg}.
Further, this second result stands for almost every starting point $W_0$, and is for instance true with probability one when the initialization is drawn from a standard continuous distribution.

\begin{restatable}[Muon cannot converge: adaptive stepsize and low momentum]{counterexample}{MaincounterExampleBeta} 
\label{Cex:offline stepsizes momentum}  
Consider  $\beta \in \left [0, \tfrac{1}{2}\right)$.
Let $f$ be as in \Cref{def:counterexample matrix kinky} with $c = \tfrac{1}{2} - \beta$.
Consider the iterates $W_t$ of~\eqref{eq:muon} or \eqref{eq:muon-reg} applied to $f$ where the stepsizes 
$\lambda_t$ are chosen as an offline function of the previous subgradients $G_0, \ldots, G_t$. 
Then there exists a set of zero measure such that, for any initialization outside of
this set, the iterates cannot converge to a minimizer of $f$.   
\end{restatable}

Counterexample~\ref{Cex:offline stepsizes momentum} shows that even a local convergence theory for Muon is not possible for convex and Lipschitz functions. 
More precisely, when $\beta \in [0,\tfrac{1}{2})$, almost every initialization in a ball centered at a minimizer $W^\star$ will fail to converge, no matter how small is the radius of this ball.
This applies in particular to \eqref{eq:muon-const}, which corresponds to the case $\beta = 0$.
There is a striking difference with the smooth setting, where \eqref{eq:muon-const} is known to converge with a suitable choice of stepsizes.

\section{Fixing Muon on Lipschitz nonsmooth functions}\label{sec-muon-ef}

Counterexample~\ref{thm-counterex-dec-steps} 
and~Counterexample~\ref{Cex:offline stepsizes momentum} show that  Muon does not converge on the class of convex and Lipschitz functions. 
Here we propose to modify Muon by incorporating an error feedback structure, which leads to an algorithm (named EF-Muon) enjoying global convergence for convex and Lipschitz functions.
Our approach has an even broader scope: we show that error feedback allows us to guarantee convergence \emph{for any non-Euclidean} subgradient descent scheme incorporating momentum. 

\subsection{From subgradient descent to compressed gradients via LMOs}

For any chosen norm $\| \cdot\|$, the regularized subgradient descent~\eqref{eq:reg-gd} combined with momentum can be equivalently rewritten \cite[Proposition 1]{anthology} in compact form
\begin{eqnarray}
        M_t &=& \beta M_{t-1} + (1- \beta) G_t, \quad G_t \in \partial f(W_t), \label{algo:sharp GD momentum} \\
         W_{t+1} &\in & W_t - \lambda_t   \|M_t\|_* \LMO_{\|\cdot\|}(M_t), \nonumber
\end{eqnarray}
where the \emph{linear maximization oracle} (LMO) and the \emph{dual norm} are respectively defined as
\begin{eqnarray*}
    \LMO_{\|\cdot\|}(M) &:=& \argmax{X \in \R^{m\times n} : \|X\|\leq 1} \langle X, M \rangle, 
    \\
    \|M\|_* &:=& \max_{X \in \R^{m\times n} : \|X\|\leq 1}  \langle X, M \rangle.  
\end{eqnarray*}
Formally the LMO is a set-valued operator, so in practice one must consider a particular selection in $\LMO_{\|\cdot\|}(M)$, a standard choice being the element with the least norm. 
For instance, when considering the operator norm, the least norm element of $\LMO_{\|\cdot\|_{\operatorname{op}}}(M)$ is exactly the polar factor $\polar(M)$, see~\Cref{lem-lmo-op-norm-least-norm-polar}. 
This illustrates that \eqref{eq:muon-reg} is a particular case of \eqref{algo:sharp GD momentum} for the operator norm.

The main operation in \eqref{algo:sharp GD momentum} is the computation of 
\begin{equation}\label{D:sharp operator}
    W \mapsto \|W\|_* \LMO_{\|\cdot\|}(W),
\end{equation}
which is also known (up to a change of sign) as the \emph{sharp operator} \cite{Nes12a}.
We show next that this sharp operator is, up to a multiplicative constant depending on the dimension of the problem, a \emph{compression} operator \cite{BezHorRicSaf23}.
This fact has been well-known for various choices of norms, 
but seems to be new when considering any arbitrary norm on a vector space.

\begin{restatable}[Dual subgradient compression operator]{proposition}{DualSubgradientCompression}
\label{L:LMO gives compressor operator}
Let $\|\cdot\|$ be any norm on $\mathbb{R}^{m\times n}$ with dual norm $\|\cdot\|_*$. 
Let $\alpha,\beta>0$ satisfy for all $W \in \mathbb{R}^{m\times n}$
\begin{equation} \label{eq:nor-equiv-alpha-beta}
\alpha \|W\|_F \le \|W\| \le \beta \|W\|_F.
\end{equation}
Then the operator $\mathcal{C}(W) := \alpha^2 \|W\|_*\LMO_{\|\cdot\|}(W)$
is a compression operator
, that is
\begin{equation} \label{eq:compressor-delta-alpha-beta}
\|W- \mathcal{C}(W)\|_F^2
\leq
\bigl(1-\delta\bigr)\|W\|_F^2,
\qquad
\delta=\frac{\alpha^2}{\beta^2}.
\end{equation}
In particular with the operator norm $\|\cdot\|_{\operatorname{op}}$ we can take $\alpha^2 = \delta = 1/r$ with $r = \min\{m,n\}$.
\end{restatable}

The above lemma allows us to draw a connection between Muon and the literature on compressed algorithms \cite{BezHorRicSaf23,DemMalSokRic23}.
Compressed algorithms are designed for solving optimization problems on distributed systems, which is a completely separate problem from ours.
This connection is nevertheless useful because this literature shares a similar point with us: it is known that some compressors may introduce such a bias that convergence may fail \cite{BezHorRicSaf23}, reminiscent of what we have learned in \Cref{sec-counterex-muon}.
More interesting, strategies were proposed for restoring convergence by modifying the algorithm.
One of them is the use of \emph{error feedback} \cite{SeiFuDroLiYu14,StiCorJag18,BezHorRicSaf23} which we detail and use next.

\subsection{Error feedback on top of non-Euclidean SGD with momentum}

Roughly speaking, error feedback  consists in keeping track of the discrepancy between a true direction $M_t$ and its compression $\mathcal{C}(M_t)$, and to use it in a later iteration so as to rectify the trajectory of the iterates.
We propose to incorporate error feedback into Muon, and more generally into \eqref{algo:sharp GD momentum} for any choice of norm, which gives rise to the EF-M algorithm below.
We do so following the ideas from \cite{pmlr-v97-karimireddy19a}  where error feedback was proposed to restore the convergence of sign subgradient descent, which corresponds in our framework to the choices $\Vert \cdot \Vert = \Vert \cdot \Vert_\infty$ and $\beta=0$.

\begin{algorithm}
\caption{Error feedback with momentum (EF-M)}
\label{alg:efm}
\begin{algorithmic}[1]
\Require stepsizes $\lambda_t>0$, momentum parameter $\beta\in[0,1)$, compressor $\cC:\R^{m\times n}\to\R^{m\times n}$
\State Initialize $W_0\in\R^{m\times n}$, $E_0=0$, and $M_{-1}=0$
\For{$t=0,1,\dots,T$}
    \State $G_t \in \R^{m\times n}$ \Comment{Sample a stochastic subgradient } 
    \State $
        M_{t}=\beta M_{t-1}+(1-\beta)G_t
    $ \Comment{ Update the momentum} 
    \State
    $
        P_t=E_t+\lambda_t M_{t}
    $  \Comment{Form the error-corrected message} 
    \State  
    $
        W_{t+1}=W_t-\cC(P_t)
    $ \Comment{Update the iterate with compressed weights}
    \State 
    $
        E_{t+1}=P_t-\cC(P_t)
    $ \Comment{Update the error feedback memory}
\EndFor
\end{algorithmic}
\end{algorithm}

We can use \Cref{L:LMO gives compressor operator} to instantiate \Cref{alg:efm} for various choices of norm.
When considering the operator norm $\Vert \cdot \Vert_{\operatorname{op}}$ on matrices, taking 
\begin{equation*}
    \mathcal{C}(W) = \frac{1}{\min\{m,n\}} \Vert W \Vert_{\nuc} \polar(W)
\end{equation*}
turns EF-M into EF-Muon, see \Cref{S:LMO compressors basic examples}.
The choice of the norm $\Vert \cdot \Vert_\infty$ on vectors gives EF-signSGD \cite{pmlr-v97-karimireddy19a}.
Since our framework is flexible enough to handle any normed vector space, we can also very easily specify \Cref{alg:efm} for product spaces of matrices associated with different layers.
We will use such variant of EF-Muon for our experiments in \Cref{sec-main-experiments}, and the corresponding technical details can be found in \Cref{sec:Cartesian-compressor} and \Cref{appdx-muonmax}.

Thanks to the error feedback mechanism in~\Cref{alg:efm}, we are able to restore the convergence of every non-Euclidean stochastic momentum method (e.g., spectral descent and Muon), on the class of convex and Lipschitz functions.
In \Cref{fig-counterex-muon-trajectory} we illustrate the convergence of EF-Muon on our counterexample from~\Cref{def:counterexample matrix kinky}.

\begin{restatable}[Anytime convex Lipschitz convergence of EF-M]{theorem}{MainTheoremEFM}
\label{thm:efm-convex}
Let  $f:\R^{m\times n}\to\R$ be convex and let $W^\star$ be a minimizer.
Let $\cC$ be a compression operator, meaning that for all $W\in\R^{m\times n}$,
\[
    \|\cC(W)-W\|_F^2\leq (1-\delta)\|W\|_F^2,
    \qquad \delta\in(0,1].
\]
Consider the iterates of \Cref{alg:efm}, and assume  that $G_t$ is a stochastic subgradient satisfying 
\[
    \mathbb{E}_t[G_t]\in \partial f(W_t),
    \qquad 
    \mathbb{E}_t\|G_t\|_F^2\leq \sigma^2.
\]
If we consider a vanishing schedule for the stepsizes $\lambda_t = 1/\sqrt{t+1}$, then
\[
    \mathbb{E} \left[ f(\bar{W}_T) \right] - f(W^\star) \leq \frac{\|W_0-W^\star\|_F^2}{2 \sqrt{T+1}} + \sigma^2 \left( \frac{2 \sqrt{1-\delta}}{\delta} + \frac{\beta}{1-\beta} + \frac{1}{2} \right) \frac{1 + \log(T+1)}{\sqrt{T+1}},
\]
where $\bar W_T=\frac{1}{T+1}\sum_{t=0}^T W_t$.
\end{restatable}

In the Euclidean setting (i.e., $\mathcal{C} =\,$ identity  and $\delta = 1$),
\Cref{thm:efm-convex} recovers the standard convergence rate of SGD.
For general compressors, but turning off the momentum ($\beta=0$), \Cref{thm:efm-convex} recovers the convergence rate from  
\cite[Theorem III]{pmlr-v97-karimireddy19a}.
A more general result can be found in the appendix, allowing for any choice of stepsizes $\lambda_t$.
In the particular case of EF-Muon, we observe that \Cref{L:LMO gives compressor operator} suggests that the stepsizes should scale like $1/r$, where $r = \min\{m,n\}$.
This is in line with \cite[Theorem 4.6]{shen2025convergenceanalysisofmuon}, which is the only known convergence result for \eqref{eq:muon-const}, under the assumption that the stepsizes scale like $1/r$.

\section{Experiments}\label{sec-main-experiments}
In this section we aim to experimentally study the effects of error feedback on
Muon optimizer, by comparing Muon with EF-M method in~\Cref{alg:efm}.  
The results in the preceding sections were stated for a single parameter matrix. In applications, however, the loss function depends on a product space of parameters
typically matrices and vectors associated with different layers. We choose a norm over this product space (see~\Cref{sec:Cartesian-compressor} for details) so that the resulting regularized subgradient method~\eqref{algo:sharp GD momentum} gives a variant of the MuonMax method~\cite{Crawshaw2025,anthology}. MuonMax was in fact the first proposed variant of Muon~\cite{anthology}, and we choose this variant because in a systematic study over all variants of Muon, MuonMax was found to perform best among the regularized variants of Muon~\cite{Crawshaw2025}.

We then compare  MuonMax to   EF-MuonMax: the result of  our EF-M method  when paired with the same norm that results in MuonMax, 
see \Cref{appdx-muonmax}. We also benchmark against vanilla Muon which is given by~\eqref{eq:muon} on the matrix layers, and is often the most competitive variant of Muon~\cite{modded_nanogpt_2024}.
 All experiments use the \texttt{init2winit}~\cite{init2winit2026github} codebase on TPU v5 lite
pods with $2\times 4$ topology. We briefly describe both experiments below and leave more details to \Cref{sec-numerical-experiments} in the appendix. Each variant of Muon also needs to be combined with a \emph{paired optimizer}, which is used for updating the non-matrix parameters (e.g., biases). The choice of the paired optimizer depends on the task.

\textbf{Image classification.} We train a WideResNet-28-10 with 36M parameters
on CIFAR-10 dataset.
Both MuonMax and EF-MuonMax use signed momentum as the paired optimizer, 
while vanilla Muon uses SGD with momentum. We run 60 trials per method over different hyperparameter choices, and then choose the best resulting run. See~\Cref{sec:image-class} for details.

\textbf{Language modeling.} We train a nanoGPT model with 206M parameters on the FineWeb-Edu 10B dataset~\cite{penedo2024fineweb}.
We use NAdamW as the paired optimizer for all methods. 
For MuonMax and EF-MuonMax we use 30 trials per method to tune the hyperparameters. For vanilla Muon, we used the default hyperparameters from \texttt{init2winit}. These were tuned separately following the procedure outlined in~\cite{medapati_trai_neur_netw25}. We give more details in \Cref{sec-lang-modeling}.

\textbf{Results.}
The results are shown in \Cref{tab:combined-results} and 
\Cref{fig:finewebedu-nadam}.
On image classification, MuonMax achieves the lowest validation CE loss of $0.265$, slightly outperforming vanilla Muon $0.287$, 
while EF-MuonMax trails both at $0.388$.
On language modeling, the ordering shifts: vanilla Muon reaches the best validation CE loss of $3.01$, followed by MuonMax with $3.27$ and EF-MuonMax with $3.43$.
In both settings, EF-MuonMax is worse than MuonMax.

\begin{table}[h]
\centering
\small
\begin{tabular}{llccc}
\toprule
Task & Method & Learning rate & LR multiplier & Validation CE \\
\midrule
\multirow{3}{*}{Image classification}
& Vanilla Muon & $10^{-4}$            & $1$   & $0.287$ \\
& MuonMax      & $5\!\times\!10^{-4}$ & $100$ & \textbf{0.265} \\
& EF-MuonMax   & $5\!\times\!10^{-3}$ & $30$  & $0.388$ \\
\midrule
\multirow{3}{*}{Language modeling}
& Vanilla Muon & $1.36\!\times\!10^{-4}$ & $28.6$ & \textbf{3.01} \\
& MuonMax      & $0.01$                  & $30$   & $3.27$ \\
& EF-MuonMax   & $0.01$                  & $30$   & $3.43$ \\
\bottomrule
\end{tabular}
\vspace{0.5em}
\caption{Best configurations for each method on image classification and language modeling tasks. }
\label{tab:combined-results}
\end{table}

\begin{figure}[h!]
  \centering 
      \centering
      \includegraphics[width=\linewidth]{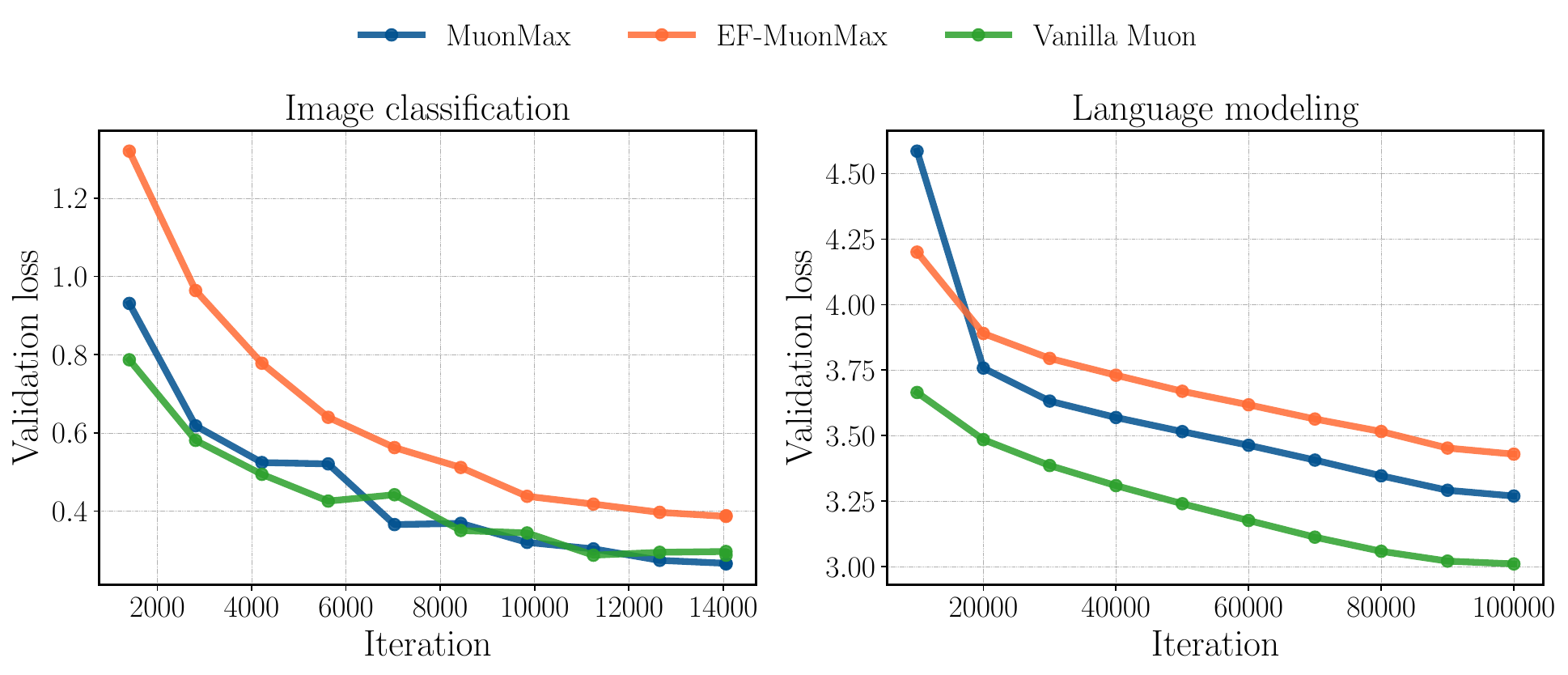}
      \caption{Validation CE loss of vanilla Muon, MuonMax and EF-MuonMax. Left: WideResNet-28-10 with 36M parameters on CIFAR-10. Right: nanoGPT with 206M parameters on FineWeb-Edu 10B.}
      \label{fig:finewebedu-nadam} 
\end{figure}

\paragraph{Discussion.}
These experiments do not support the error-feedback interpretation. Error
feedback fails to improve MuonMax on both tasks, contrary to what the convex
counterexamples and \Cref{thm:efm-convex} would suggest.
Therefore, it is not useful to view the LMO as a compression operator
whose error should be corrected by the feedback mechanism.

\section{Conclusions and future directions}

Our work is analogous in spirit to the early analysis of Adam by~\cite{ReddiKK18}, which showed that Adam can fail to converge on convex Lipschitz functions and proposed AMSGrad as a modification that restores convergence guarantees. In practice, however, AMSGrad has not displaced Adam, plausibly because its theoretical advantage does not consistently translate into empirical gains. We find a similar phenomenon for Muon: standard Muon can fail to converge on convex Lipschitz objectives, while an error-feedback modification restores a convergence guarantee but does not improve performance in our experiments. This suggests that worst-case convex Lipschitz theory, although useful as a diagnostic, is not by itself the right model for explaining Muon's empirical success.

Our negative results therefore point to the need for a different theoretical testbed for Muon. Convex smoothness is one natural candidate, and has already led to useful insights~\cite{davis2025spectral}; alternatively, the relevant assumptions may involve other forms of structure that are absent from the nonsmooth Lipschitz model. Regardless of the eventual model class, the reduction in~\Cref{thm-muon-reduction-informal} provides a useful diagnostic tool: any proposed explanation of Muon should be consistent with its exact specialization to signed momentum on diagonal matrices.

\section{Limitations} \label{sec:limitations}
Each counterexample comes with certain limitations. Counterexample~\ref{thm-counterex-dec-steps} requires decreasing stepsizes. This is fine for the variant of Muon in~\eqref{eq:muon}, which is the version most used in practice~\cite{modded_nanogpt_2024}, but it excludes the regularized variant of Muon~\eqref{eq:muon-reg}. Whereas our Counterexample~\ref{Cex:offline stepsizes momentum} holds for adaptive stepsize sequences, including \eqref{eq:muon-reg}, but not for the most used setting of $\beta$, since it only considers $\beta \in [0, \tfrac{1}{2}).$ A further limitation is that neither counterexample considers a scheduler for $\beta$. Some of the most interesting proofs of convergence for momentum in the stochastic setting hold for $\beta_t \uparrow 1$, see \cite{GarGow23}. Furthermore, though our counterexamples rule out a possible convergence result, they do not exclude the possibility of a complexity result, see~\Cref{asec:conv-complex-limit} for details of this distinction.
Finally, our work did not consider the approximate calculation of the polar factor~\cite{amsel2025polar,muon}, or using a Nesterov style momentum instead of Polyak style momentum.


\renewcommand*{\bibfont}{}
{ 
\printbibliography

@article{schaipp2025optimizationbenchmarkdiffusionmodels,
  title   = {Optimization benchmark for diffusion models on dynamical systems},
  author  = {Fabian Schaipp},
  year    = {2025},
  journal = {arXiv preprint arXiv:2510.19376},
  url     = {https://arxiv.org/abs/2510.19376}
}

@inproceedings{gruntkowska2025errorfeedbackmuonfriends,
  title     = {Error feedback for {{Muon}} and friends},
  author    = {Gruntkowska, Kaja and Gaponov, Alexander and Tovmasyan, Zhirayr and Richt{\'a}rik, Peter},
  booktitle = {The Fourteenth International Conference on Learning Representations},
  year      = {2026}
}

@book{Pey,
  title     = {Convex optimization in normed spaces},
  author    = {Peypouquet, Juan},
  year      = {2015},
  series    = {{{SpringerBriefs}} in {{Optimization}}},
  publisher = {Springer International Publishing},
  address   = {Cham},
  urldate   = {2015-04-24},
  isbn      = {978-3-319-13709-4 978-3-319-13710-0}
}

@article{adagrad,
  author    = {Duchi, John and Hazan, Elad and Singer, Yoram},
  title     = {Adaptive subgradient methods for online learning and stochastic optimization},
  year      = {2011},
  publisher = {JMLR.org},
  volume    = {12},
  journal   = {J. Mach. Learn. Res.},
  month     = jul,
  pages     = {2121–2159}
}

@book{golub13,
  author    = {Golub, Gene H. and van Loan, Charles F.},
  edition   = {Fourth},
  isbn      = {1421407949 9781421407944},
  publisher = {JHU Press},
  title     = {Matrix computations},
  url       = {http://www.cs.cornell.edu/cv/GVL4/golubandvanloan.htm},
  year      = {2013}
}

@article{yang2024spectralconditionfeaturelearning,
  title   = {A spectral condition for feature learning},
  author  = {Greg Yang and James B. Simon and Jeremy Bernstein},
  year    = {2024},
  journal = {arXiv preprint arXiv:2310.17813},
  url     = {https://arxiv.org/abs/2310.17813}
}

@article{ba2016layernormalization,
  title   = {Layer normalization},
  author  = {Jimmy Lei Ba and Jamie Ryan Kiros and Geoffrey E. Hinton},
  year    = {2016},
  journal = {arXiv preprint arXiv:1607.06450},
  url     = {https://arxiv.org/abs/1607.06450}
}

@inproceedings{balles2018dissecting,
  title     = {Dissecting {{Adam}}: The sign, magnitude and variance of stochastic gradients},
  author    = {Lukas Balles and Philipp Hennig},
  booktitle = {International Conference on Learning Representations},
  year      = {2018},
  url       = {https://openreview.net/forum?id=S1EwLkW0W}
}

@inproceedings{zhang2019root,
  title     = {Root mean square layer normalization},
  author    = {Zhang, Biao and Sennrich, Rico},
  booktitle = {Advances in Neural Information Processing Systems (NeurIPS)},
  year      = {2019}
}

@inproceedings{pmlr-v97-karimireddy19a,
  title     = {Error feedback fixes {{SignSGD}} and other gradient compression schemes},
  author    = {Karimireddy, Sai Praneeth and Rebjock, Quentin and Stich, Sebastian and Jaggi, Martin},
  booktitle = {Proceedings of the 36th International Conference on Machine Learning},
  pages     = {3252--3261},
  year      = {2019},
  volume    = {97},
  series    = {Proceedings of Machine Learning Research},
  publisher = {PMLR}
}

@inproceedings{orvieto2025in,
  title     = {In search of {{Adam}}'s secret sauce},
  author    = {Antonio Orvieto and Robert M. Gower},
  booktitle = {The Thirty-ninth Annual Conference on Neural Information Processing Systems},
  year      = {2025},
  url       = {https://openreview.net/forum?id=CH72XyZs4y}
}

@inproceedings{vyas2024soap,
  title     = {{{SOAP}}: Improving and stabilizing {{Shampoo}} using {{Adam}}},
  author    = {Nikhil Vyas and Depen Morwani and Rosie Zhao and Mujin Kwun and Itai Shapira and David Brandfonbrener Imber and Lucas Janson and Sham Kakade},
  booktitle = {Advances in Neural Information Processing Systems},
  volume    = {37},
  pages     = {43571--43602},
  year      = {2024}
}

@inproceedings{martens2015optimizing,
  title        = {Optimizing neural networks with {{Kronecker}}-factored approximate curvature},
  author       = {James Martens and Roger Grosse},
  booktitle    = {International Conference on Machine Learning},
  pages        = {2408--2417},
  year         = {2015},
  organization = {PMLR}
}

@InProceedings{pmlr-v80-gupta18a,
  title     = {{{Shampoo}}: Preconditioned stochastic tensor optimization},
  author    = {Vineet Gupta and Tomer Koren and Yoram Singer},
  booktitle = {Proceedings of the 35th International Conference on Machine Learning},
  pages     = {1842--1850},
  year      = {2018},
  volume    = {80},
  series    = {Proceedings of Machine Learning Research},
  publisher = {PMLR},
  url       = {https://proceedings.mlr.press/v80/gupta18a.html}
}

@article{StiCorJag18,
  title   = {Sparsified {{SGD}} with memory},
  author  = {Stich, Sebastian U. and Cordonnier, Jean-Baptiste and Jaggi, Martin},
  year    = {2018},
  journal = {Advances in Neural Information Processing Systems},
  volume  = {31}
}

@inproceedings{SeiFuDroLiYu14,
  title     = {1-bit stochastic gradient descent and its application to data-parallel distributed training of speech {{DNNs}}},
  booktitle = {Interspeech},
  author    = {Seide, Frank and Fu, Hao and Droppo, Jasha and Li, Gang and Yu, Dong},
  year      = {2014},
  volume    = {2014},
  pages     = {1058--1062},
  publisher = {Singapore}
}

@InProceedings{pmlr-v267-schaipp25a,
  title     = {The surprising agreement between convex optimization theory and learning-rate scheduling for large model training},
  author    = {Fabian Schaipp and Alexander H{\"a}gele and Adrien Taylor and Umut Simsekli and Francis Bach},
  booktitle = {Proceedings of the 42nd International Conference on Machine Learning},
  pages     = {53267--53294},
  year      = {2025},
  volume    = {267},
  series    = {Proceedings of Machine Learning Research},
  publisher = {PMLR},
  url       = {https://proceedings.mlr.press/v267/schaipp25a.html}
}

@inproceedings{ReddiKK18,
  author    = {Sashank J. Reddi and Satyen Kale and Sanjiv Kumar},
  booktitle = {International Conference on Learning Representations},
  title     = {On the convergence of {{Adam}} and beyond},
  year      = {2018}
}

@article{anthology,
  title   = {Old optimizer, new norm: An anthology},
  author  = {Jeremy Bernstein and Laker Newhouse},
  year    = {2024},
  journal = {arXiv preprint arXiv:2409.20325}
}

@misc{muon,
  author       = {Keller Jordan and Yuchen Jin and Vlado Boza and Jiacheng You and Franz Cesista and Laker Newhouse and Jeremy Bernstein},
  title        = {{Muon}: An optimizer for hidden layers in neural networks},
  year         = {2024},
  howpublished = {Blog post, retrieved May 2026.},
  url          = {https://kellerjordan.github.io/posts/muon/}
}

@misc{deepseekai2026deepseekv4,
  title        = {{DeepSeek-V4}: Towards Highly Efficient Million-Token Context Intelligence},
  author       = {{DeepSeek-AI}},
  year         = {2026},
  howpublished = {\url{https://huggingface.co/deepseek-ai/DeepSeek-V4-Pro/blob/main/DeepSeek_V4.pdf}},
  note         = {Technical report, hosted on Hugging Face}
}

@inproceedings{scion,
  title     = {Training deep learning models with norm-constrained {{LMOs}}},
  author    = {Pethick, Thomas and Xie, Wanyun and Antonakopoulos, Kimon and Zhu, Zhenyu and Silveti-Falls, Antonio and Cevher, Volkan},
  booktitle = {Proceedings of the 42nd International Conference on Machine Learning},
  series    = {Proceedings of Machine Learning Research},
  volume    = {267},
  pages     = {53230--53266},
  publisher = {PMLR},
  year      = {2025}
}

@inproceedings{adam,
  title     = {{Adam}: {A} Method for Stochastic Optimization},
  author    = {Kingma, Diederik P. and Ba, Jimmy},
  booktitle = {International Conference on Learning Representations (ICLR)},
  year      = {2015}
}

@InProceedings{carlson2015stochastic,
  title     = {Stochastic spectral descent for restricted {{Boltzmann}} machines},
  author    = {Carlson, David and Cevher, Volkan and Carin, Lawrence},
  booktitle = {Proceedings of the Eighteenth International Conference on Artificial Intelligence and Statistics},
  pages     = {111--119},
  year      = {2015},
  volume    = {38},
  series    = {Proceedings of Machine Learning Research},
  publisher = {PMLR}
}

@inproceedings{carlson2015preconditioned,
  author    = {Carlson, David E. and Collins, Edo and Hsieh, Ya-Ping and Carin, Lawrence and Cevher, Volkan},
  booktitle = {Advances in Neural Information Processing Systems},
  publisher = {Curran Associates, Inc.},
  title     = {Preconditioned spectral descent for deep learning},
  volume    = {28},
  year      = {2015}
}

@article{davis2025spectral,
  title   = {When do spectral gradient updates help in deep learning?},
  author  = {Davis, Damek and Drusvyatskiy, Dmitriy},
  journal = {arXiv preprint arXiv:2512.04299},
  year    = {2025}
}

@article{Crawshaw2025,
  author  = {Michael Timothy Crawshaw and Chirag Modi and Mingrui Liu and Robert M. Gower},
  title   = {An exploration of {{Non-Euclidean}} gradient descent: {{Muon}} and its many variants},
  journal = {arXiv preprint arXiv:2510.09827},
  year    = {2025}
}

@article{moonlight,
  title   = {{{Muon}} is scalable for {{LLM}} training},
  author  = {Liu, Jingyuan and Su, Jianlin and Yao, Xingcheng and Jiang, Zhejun and Lai, Guokun and Du, Yulun and Qin, Yidao and Xu, Weixin and Lu, Enzhe and Yan, Junjie and others},
  journal = {arXiv preprint arXiv:2502.16982},
  year    = {2025}
}

@article{polargrad,
  title   = {{{PolarGrad}}: A class of matrix-gradient optimizers from a unifying preconditioning perspective},
  author  = {Lau, Tim Tsz-Kit and Long, Qi and Su, Weijie},
  journal = {arXiv preprint arXiv:2505.21799},
  year    = {2025}
}

@article{kovalev2025understanding,
  title   = {Understanding gradient orthogonalization for deep learning via {{Non-Euclidean}} trust-region optimization},
  author  = {Kovalev, Dmitry},
  journal = {arXiv preprint arXiv:2503.12645},
  year    = {2025}
}

@inproceedings{penedo2024fineweb,
  title     = {The {{FineWeb}} datasets: Decanting the web for the finest text data at scale},
  author    = {Guilherme Penedo and Hynek Kydl{\'\i}{\v{c}}ek and Loubna Ben allal and Anton Lozhkov and Margaret Mitchell and Colin Raffel and Leandro Von Werra and Thomas Wolf},
  booktitle = {The Thirty-eight Conference on Neural Information Processing Systems Datasets and Benchmarks Track},
  year      = {2024},
  url       = {https://openreview.net/forum?id=n6SCkn2QaG}
}

@misc{modded_nanogpt_2024,
  title        = {{Modded-nanoGPT}: Speedrunning the {nanoGPT} baseline},
  author       = {Jordan, Keller and Bernstein, Jeremy and Rappazzo, Ben and Boza, Vlado and You, Jiacheng and Cesista, Franz and Koszarsky, Braden},
  year         = {2024},
  howpublished = {\url{https://github.com/KellerJordan/modded-nanogpt}},
  note         = {GitHub repository}
}

@misc{nanochat_challenge,
  author = {Andrej Karpathy},
  title = {nanochat: The best ChatGPT that \$100 can buy},
  year = {2025},
  publisher = {GitHub},
  url = {https://github.com/karpathy/nanochat}
}

@inproceedings{amsel2025polar,
  title     = {The {{Polar}} {{Express}}: Optimal matrix sign methods and their application to the {{Muon}} algorithm},
  author    = {Noah Amsel and David Persson and Christopher Musco and Robert M. Gower},
  booktitle = {The Fourteenth International Conference on Learning Representations},
  year      = {2026}
}

@article{krizhevsky2009learning,
  title     = {Learning multiple layers of features from tiny images},
  author    = {Krizhevsky, Alex and Hinton, Geoffrey and others},
  year      = {2009},
  publisher = {Toronto, ON, Canada}
}

@article{su2025isotropiccurvaturemodelunderstanding,
  title   = {Isotropic curvature model for understanding deep learning optimization: Is gradient orthogonalization optimal?},
  author  = {Weijie Su},
  year    = {2025},
  journal = {arXiv preprint arXiv:2511.00674},
  url     = {https://arxiv.org/abs/2511.00674}
}

@article{shen2025convergenceanalysisofmuon,
  title   = {On the convergence analysis of {{Muon}}},
  author  = {Wei Shen and Ruichuan Huang and Minhui Huang and Cong Shen and Jiawei Zhang},
  year    = {2025},
  journal = {arXiv preprint arXiv:2505.23737},
  url     = {https://arxiv.org/abs/2505.23737}
}

@article{li2025noteonconvergenceofmuon,
  title   = {A note on the convergence of {{Muon}}},
  author  = {Jiaxiang Li and Mingyi Hong},
  year    = {2025},
  journal = {arXiv preprint arXiv:2502.02900},
  url     = {https://arxiv.org/abs/2502.02900}
}

@article{sato2025criticalbatchmuon,
  title   = {Convergence bound and critical batch size of {{Muon}} optimizer},
  author  = {Sato, Naoki and Naganuma, Hiroki and Iiduka, Hideaki},
  year    = {2025},
  journal = {arXiv preprint arXiv:2507.01598},
  url     = {https://arxiv.org/abs/2507.01598}
}

@article{sfyraki2025lionsandmuons,
  title   = {Lions and {{Muon}}s: Optimization via stochastic {{Frank}}--{{Wolfe}}},
  author  = {Sfyraki, Maria-Eleni and Wang, Jun-Kun},
  year    = {2025},
  journal = {arXiv preprint arXiv:2506.04192},
  url     = {https://arxiv.org/abs/2506.04192}
}

@article{chang2025muonandbeyond,
  title   = {On the convergence of {{Muon}} and beyond},
  author  = {Chang, Da and Liu, Yongxiang and Yuan, Ganzhao},
  year    = {2025},
  journal = {arXiv preprint arXiv:2509.15816},
  url     = {https://arxiv.org/abs/2509.15816}
}

@inproceedings{xie2025tale,
  title     = {A tale of two geometries: Adaptive optimizers and {{Non-Euclidean}} descent},
  author    = {Xie, Shuo and Wang, Tianhao and Wu, Beining and Li, Zhiyuan},
  booktitle = {The Fourteenth International Conference on Learning Representations},
  year      = {2026}
}

@article{riabinin2026does,
  title   = {Where does warm-up come from? Adaptive scheduling for norm-constrained optimizers},
  author  = {Riabinin, Artem and Veprikov, Andrey and Bolatov, Arman and Tak{\'a}{\v{c}}, Martin and Beznosikov, Aleksandr},
  journal = {arXiv preprint arXiv:2602.05813},
  year    = {2026}
}

@article{jiang2026adaptive,
  title   = {Adaptive matrix online learning through smoothing with guarantees for nonsmooth nonconvex optimization},
  author  = {Jiang, Ruichen and Mhammedi, Zakaria and Mohri, Mehryar and Mokhtari, Aryan},
  journal = {arXiv preprint arXiv:2602.08232},
  year    = {2026}
}

@article{hazan2016introduction,
  title     = {Introduction to online convex optimization},
  author    = {Hazan, Elad},
  journal   = {Foundations and Trends in Optimization},
  volume    = {2},
  number    = {3-4},
  pages     = {157--325},
  year      = {2016},
  publisher = {Emerald Publishing Limited}
}

@article{tieleman2012lecture,
  title   = {Lecture 6.5-{{RMSProp}}, {{Coursera}}: Neural networks for machine learning},
  author  = {Tieleman, Tijmen and Hinton, Geoffrey},
  journal = {University of Toronto, Technical Report},
  volume  = {6},
  year    = {2012}
}

@article{GarGow23,
  title   = {Handbook of convergence theorems for (stochastic) gradient methods},
  author  = {Garrigos, Guillaume and Gower, Robert M.},
  year    = {2023},
  journal = {arXiv preprint arXiv:2301.11235}
}

@article{team2025kimi,
  title   = {{{Kimi K2}}: Open agentic intelligence},
  author  = {Team, Kimi and Bai, Yifan and Bao, Yiping and Charles, Y and Chen, Cheng and Chen, Guanduo and Chen, Haiting and Chen, Huarong and Chen, Jiahao and Chen, Ningxin and others},
  journal = {arXiv preprint arXiv:2507.20534},
  year    = {2025}
}

@article{watson1992characterization,
  title   = {Characterization of the subdifferential of some matrix norms},
  author  = {Watson, G Alistair},
  journal = {Linear algebra and its applications},
  volume  = {170},
  number  = {1},
  pages   = {33--45},
  year    = {1992}
}

@article{Dahl2023AlgoPerf,
  title   = {Benchmarking neural network training algorithms},
  author  = {Dahl, George E. and Schneider, Frank and Nado, Zachary and Agarwal, Naman and Sastry, Chandramouli Shama and Hennig, Philipp and Medapati, Sourabh and Eschenhagen, Runa and Kasimbeg, Priya and Suo, Daniel and Bae, Juhan and Gilmer, Justin and Peirson, Abel L. and Khan, Bilal and Anil, Rohan and Rabbat, Mike and Krishnan, Shankar and Snider, Daniel and Amid, Ehsan and Kongtao Chen and Chris J. Maddison and Rakshith Vasudev and Michal Badura and Ankush Garg and Peter Mattson},
  year    = {2023},
  journal = {arXiv preprint arXiv:2306.07179}
}

@inproceedings{Kasimbeg2025AlgoPerfResults,
  title     = {Accelerating neural network training: An analysis of the {{AlgoPerf}} competition},
  author    = {Kasimbeg, Priya and Schneider, Frank and Eschenhagen, Runa and Bae, Juhan and Sastry, Chandramouli Shama and Saroufim, Mark and Boyuan, Feng and Wright, Less and Yang, Edward Z. and Nado, Zachary and Medapati, Sourabh and Hennig, Philipp and Rabbat, Michael and Dahl, George E.},
  booktitle = {The Thirteenth International Conference on Learning Representations},
  year      = {2025},
  url       = {https://openreview.net/forum?id=CtM5xjRSfm}
}

@article{medapati_trai_neur_netw25,
  author  = {Sourabh Medapati and Priya Kasimbeg and Shankar Krishnan and Naman Agarwal and George Dahl},
  journal = {arXiv preprint arXiv:2503.03986},
  title   = {Training neural networks faster with minimal tuning using pre-computed lists of hyperparameters for {{NAdamW}}},
  url     = {https://arxiv.org/abs/2503.03986},
  year    = {2025}
}

@article{dozat2016incorporating,
  title   = {Incorporating {{Nesterov}} momentum into {{Adam}}},
  author  = {Dozat, Timothy},
  year    = {2016},
  url     = {https://openreview.net/pdf/OM0jvwB8jIp57ZJjtNEZ.pdf},
  journal = {ICLR 2016 Workshop Track}
}

@software{init2winit2026github,
  author = {Justin M. Gilmer and George E. Dahl and Zachary Nado and Priya Kasimbeg and Sourabh Medapati and Ahmed Khaled},
  title = {{init2winit}: a JAX codebase for initialization, optimization, and tuning research},
  url = {http://github.com/google/init2winit},
  version = {0.0.3},
  year = {2026},
}

@article{Nes12a,
  title     = {Efficiency of coordinate descent methods on huge-scale optimization problems},
  author    = {Nesterov, Yurii},
  year      = {2012},
  month     = jan,
  journal   = {SIAM Journal on Optimization},
  volume    = {22},
  number    = {2},
  pages     = {341--362},
  publisher = {Society for Industrial and Applied Mathematics}
}

@article{BezHorRicSaf23,
  title   = {On biased compression for distributed learning},
  author  = {Beznosikov, Aleksandr and Horv{\'a}th, Samuel and Peter Richt{\'a}rik and Safaryan, Mher},
  year    = {2023},
  journal = {Journal of Machine Learning Research},
  volume  = {24},
  number  = {276},
  pages   = {1--50}
}

@article{DemMalSokRic23,
  title   = {A guide through the zoo of biased {{SGD}}},
  author  = {Demidovich, Yury and Malinovsky, Grigory and Sokolov, Igor and Peter Richt{\'a}rik},
  year    = {2023},
  journal = {Advances in Neural Information Processing Systems},
  volume  = {36},
  pages   = {23158--23171}
}
}


\clearpage
\appendix
\part*{Supplementary Material}

\tableofcontents
\clearpage
\section{Prior work}
\label{sec:prior-work}

Here we provide a more extensive review of prior work on Muon and spectral descent. 
\paragraph{Spectral descent and Muon.}
General idea behind Muon~\cite{muon} is based on spectral descent, which corresponds to taking the update using steepest descent over spectral norm ball instead of Euclidean norm~\cite{carlson2015stochastic,carlson2015preconditioned}.
Recent work constructs modular norms which allow one to apply steepest descent under different norms to neural networks with many layers~\cite{anthology}.
Muon and its variants, e.g., Scion and PolarGrad, can also be interpreted through the same lens~\cite{muon,scion,polargrad}.

\paragraph{Smooth analyses.}
Existing convergence results for Muon rely on smoothness assumption.
In~\cite{davis2025spectral} authors study conditions under which a spectral update gives a larger one-step decrease than a Euclidean update. 
Su~\cite{su2025isotropiccurvaturemodelunderstanding} analyzes orthogonalized updates through an isotropic curvature model and shows when spectrum homogenization and full orthogonalization are favored. 
Kovalev~\cite{kovalev2025understanding} interprets subgradient orthogonalization as a non-Euclidean trust-region method and analyzes stochastic momentum variants that recover Muon.

Several papers also prove convergence guarantees for smooth Muon variants. 
Li et al.~\cite{li2025noteonconvergenceofmuon} study both constrained and regularized formulations. 
Shen et al.~\cite{shen2025convergenceanalysisofmuon} analyze constrained Muon with and without momentum, including a star-convex setting. 
Sato et al.~\cite{sato2025criticalbatchmuon} consider momentum, Nesterov acceleration, and weight decay, with an emphasis on batch-size effects. 
Sfyraki et al.~\cite{sfyraki2025lionsandmuons} study constrained Muon together with weight decay, gradient clipping, and variance reduction. 
Chang et al.~\cite{chang2025muonandbeyond} analyze variance-reduced Muon and also obtain stronger guarantees under additional growth conditions. 
Xie et al.~\cite{xie2025tale} consider adaptive smoothness and adaptive variance in a non-Euclidean setting. 
Under a generalized smoothness assumption Takáč et al. \cite{riabinin2026does}  examined adaptive scheduling for norm-constrained optimizers, including Muon, and its implications for warm-up strategies.
One of the few references that consider the nonsmooth setting is \cite{jiang2026adaptive} where they analyzed smoothed variants of spectral descent with guarantees for nonsmooth and nonconvex optimization. 

These works are complementary to ours. 
They explain why Muon can perform well in smooth regimes; our focus is the convex and Lipschitz function class, which also includes the nonsmooth functions.

\paragraph{Convex nonsmooth analyses and adaptive methods.}
The convex Lipschitz model is one of the fundamental models that gave rise to the development and analysis of modern adaptive first-order methods.
AdaGrad~\cite{adagrad} is an adaptive subgradient method for online
learning and stochastic optimization based on follow-the-regularized-leader family of algorithms~\cite{hazan2016introduction}.
Adam~\cite{adam} combines ideas from AdaGrad 
and RMSProp~\cite{tieleman2012lecture}, and has become a standard optimizer choice in deep learning.

This perspective also extends to matrix-based optimizers. Shampoo \cite{pmlr-v80-gupta18a} can be viewed as
a matrix version of AdaGrad which replaces full-matrix preconditioner by a Kronecker product, with regret-based convergence analysis.
K-FAC~\cite{martens2015optimizing} is derived by
approximating natural-gradient methods: it uses Kronecker-factored
approximations of Fisher (or Gauss-Newton) curvature blocks to define a
preconditioned gradient step. SOAP~\cite{vyas2024soap} combines these ideas with
Adam-style adaptivity, using Shampoo's eigenspaces while applying
coordinatewise scaling in the rotated basis.

Thus, although Shampoo, K-FAC, and SOAP are matrix-based, they remain
adaptive preconditioning methods: they change the metric in which gradients are
measured.
Muon is different, as it relies on the orthogonalization of the update direction.

More recently, Schaipp et al.~\cite{pmlr-v267-schaipp25a} showed that analyses under the convex and Lipschitz model are particularly informative for understanding training dynamics in deep learning.
Therefore, it is also important to study Muon under the model that led to the development of
widely used adaptive optimizers.

\paragraph{Error feedback.}
In \cite{pmlr-v97-karimireddy19a} authors gave convex and Lipschitz counterexample for sign subgradient descent, and proved that sign subgradient descent with error feedback mechanism restores convergence. 
More recently, \cite{gruntkowska2025errorfeedbackmuonfriends} extended error feedback analysis to Muon and LMO-based methods to distributed setting under the generalized smoothness assumptions. In particular, they use compression operator for communication between workers and error feedback to guarantee convergence in distributed setting.
In contrast, we interpret LMO itself as compression operator and use error feedback to guarantee convergence in the nonsmooth setting.

\section{Proofs in \Cref{sec-counterex-muon}}

\subsection{Signed momentum}

When our parameters are vectors $w \in \R^d$, a  special case of non-Euclidean subgradient descent is given by \emph{signed subgradient descent}. 
Indeed, when we use the infinity norm $\|w\|_{\infty} := \max_{i=1,\ldots, d} |w_i|$ in~\eqref{eq:const-gd}, then one realization of the algorithm is given by
\begin{align} \label{eq:signgd}
    w_{t+1} &= w_t - \lambda_t \sign(g_t), 
    \qquad g_t \in \partial f(w_t).
    \tag{signGD}
\end{align}
Here $\sign$ is defined as the standard sign operator 
\begin{equation} \label{eq:sign0}
    \sign(x) =
\begin{cases}
+1 &   x>0\\
 0 &   x=0\\
-1 &   x<0,
\end{cases}
\end{equation}
and in what follows we will allow ourselves to apply it elementwise to vectors or matrices while keeping the same overloaded notation.
Signed subgradient descent can be combined with momentum, giving rise to the \emph{signed momentum} algorithm (or Signum), defined for some momentum parameter $\beta \in [0, 1)$ as
\begin{align}  
 m_t & = \beta m_{t-1} + (1-\beta) g_t , 
 \quad g_t \in \partial f(w_t), 
 \quad m_{-1} = 0,
 \label{eq:signum} \tag{signMomentum} 
 \\
w_{t+1} &= w_t - \lambda_t \sign(m_t). \nonumber
\end{align}
Note that \eqref{eq:signum} can also be applied to matrices $W \in \mathbb{R}^{m \times n}$ seen as vectors: in this case the sign operation $\sign(M_t)$ is applied on every coefficient of the matrix.

\subsection{A reduction of Muon to signed momentum for diagonal functions}\label{S:reduction muon matrix to vector}

Before talking about diagonal functions, let us set some definitions regarding diagonal rectangular matrices.
Given a  vector $a\in\mathbb{R}^{\min\{n,m\}}$, we define the diagonal matrix $\Diag(a) \in \mathbb{R}^{m \times n}$ via
\begin{equation*} 
\Diag(a)_{ij} = \begin{cases}
    0 & i\neq j \\
    a_i & i=j,
\end{cases} 
\qquad i=1, \ldots, m, \quad j=1, \ldots, n.
\end{equation*}
We further say that a rectangular matrix  
$A\in\mathbb{R}^{m\times n}$ is \emph{diagonal} if $A = \Diag(a)$ for some vector $a\in\mathbb{R}^{\min\{n,m\}}$. 
We also introduce the adjoint operator $\diag : \mathbb{R}^{m \times n} \to \mathbb{R}^{\min\{m,n\}}$, where $\diag(A)$ returns the vector composed of the diagonal elements of $A$.
A first important fact is that computing the polar factor of a diagonal matrix is the same as applying the $\sign$ operator component-wise to this matrix.

\begin{lemma}[Polar factor of a rectangular diagonal matrix]\label{lem:polar-diag}
 
Let $A\in\mathbb{R}^{m\times n}$ be a diagonal matrix.
Then
\[
\polar(A)=\sign(A). 
\] 
\end{lemma}

\begin{proof}[Proof of Lemma~\ref{lem:polar-diag}.] 
Let $p = \min\{m,n\}$ and $A=\Diag(a)$ for some $a \in \mathbb{R}^p$.
Let $I=\{i\in\{1,\dots,p\} \mid a_i\neq 0\}$ and $r=|I|=\operatorname{rank}(A)$.
Index $I=\{i_1,\dots,i_r\}$.
Let $e_1', \dots, e_m'$ be the canonical basis of $\mathbb{R}^m$, and $e_1, \dots, e_n$ be the canonical basis of $\mathbb{R}^n$.
Define $U\in\mathbb{R}^{m\times r}$ and $V\in\mathbb{R}^{n\times r}$ via their columns: 
\[
u_j=\sign(a_{i_j})e_{i_j}'
\in \mathbb{R}^m,
\qquad v_j=e_{i_j} \in \mathbb{R}^n,
\qquad j=1,\dots,r,
\]
and set $\Sigma=\diag(|a_{i_1}|,\dots,|a_{i_r}|)\in\mathbb{R}^{r\times r}$.
Then  
\[
U\Sigma V^\top=\sum_{j=1}^r \sign(a_{i_j})|a_{i_j}| e_{i_j}'e_{i_j}^\top
=\sum_{j=1}^r a_{i_j}e_{i_j}'e_{i_j}^\top
= A,
\]
so $A=U\Sigma V^\top$ is a reduced SVD of $A$.
Therefore, since coordinates with $a_i=0$ do not appear in the reduced SVD
\begin{eqnarray*}
    \polar(A)&=&UV^\top \\
    &=&\sum_{j=1}^r \sign(a_{i_j})e_{i_j}'e_{i_j}^\top \\
    &=& \sum_{i=1}^p \sign(a_{i})e_{i}'e_{i}^\top \\
    &=& \Diag(\sign(a_1),\dots,\sign(a_{p})) \\
    &=& \sign(A),
\end{eqnarray*}
where the last equality comes from the fact that $A$ has zero coefficients outside its diagonal, and that $\sign(0)=0$.
\end{proof}

Our main reduction result will show that~\eqref{eq:muon} and~\eqref{eq:signum} are equivalent for diagonal functions. 
This is the formal statement corresponding to our informal \Cref{thm-muon-reduction-informal}.

\begin{restatable}[Muon reduction]{lemma}{MuonReduction}
\label{lem:muon-reduction}
Let $f: \R^{m\times n} \to \R$ be a 
\emph{diagonal function}, in the sense that 
\[f(W) \; = \; f_{\diag{}}(\diag(W))\]
for some  function $f_{\diag{}}: \R^{\min\{m, n\}} \to \R$.
Assume that $f_{\diag}$ is a convex and continuous function. 
If $W_t$ are the iterates of~\eqref{eq:muon} applied to $f$, then they also follow the~\eqref{eq:signum} algorithm applied to $f$.

\end{restatable}

\begin{proof}
Let us first establish that every subgradient $G \in \partial f(W)$ is a diagonal matrix.
To see this, use the fact that $f = f_{\diag} \circ \diag$ where $f_{\diag}$ is convex continuous and $\diag : \mathbb{R}^{m \times n } \to \mathbb{R}^{\min\{m,n\}} $ is a linear operator.
That allows to write \cite[Proposition 3.28]{Pey}
\begin{equation*}
    \partial f(W) = \diag^* (\partial f_{\diag}(\diag(W))),
\end{equation*}
where the adjoint operator $\diag^*$ is exactly $\Diag$.
In other words the subgradients of $f$ belong to the range of $\Diag$, so they are indeed diagonal matrices.

At every iteration, the momentum update reads
\[
M_t = \beta M_{t-1} + (1-\beta) G_t,
\]
where $G_t \in \partial f(W_t)$ is a diagonal matrix, and the momentum matrix is initialized with $M_{-1} =0$.
By induction, $M_t$ is diagonal for every $t=0, 1,\ldots$.
Define $m_t := \diag(M_t)$ which verifies $M_t = \Diag(m_t)$.
We can then use~\Cref{lem:polar-diag} to establish 
\begin{equation*}
\polar(M_t) = \sign(M_t).
\end{equation*}
Consequently the iterates $W_t$ of~\eqref{eq:muon} verify
\begin{eqnarray*}
    M_t &=& \beta M_{t-1} + (1 - \beta) G_t, \quad G_t \in \partial f(W_t), \\
    W_{t+1} &=& W_t - \lambda_t \sign(M_t),
\end{eqnarray*}
which means that indeed $W_t$ follows the~\eqref{eq:signum} algorithm applied to $f$.

\end{proof}

\subsection{Basic results on our counterexample function from \Cref{def:counterexample matrix kinky}}

Recall that our counterexample function in \Cref{def:counterexample matrix kinky} is defined as 
\begin{equation*}
    f(W) = c \vert W_{11} + W_{22} \vert + \vert W_{11} - W_{22} \vert, \quad c \in (0,1).
\end{equation*}
It is not only a diagonal function, but only depends on the first two diagonal elements of $W$.
In view to exploit this sparse dependency on the coefficients of $W$, and for simplifying subsequent notations, we introduce 
\begin{equation*}
    \diag_2 : \mathbb{R}^{m \times n} \to \mathbb{R}^2, \quad
    \diag_2(W) = (W_{11}, W_{22}),
\end{equation*}
which is a linear operator returning the first two diagonal elements of a matrix.
Note also that $f$ is nonsmooth due to the presence of absolute values, so to compute its subgradients we will need a notation for the subgradient of the absolute value, which is the set-valued sign function
\begin{equation*}
    \Sign(x) := 
    \begin{cases}
        +1 & x>0 \\
        [-1,+1] & x = 0 \\
        -1 & x<0.
    \end{cases}
\end{equation*}
We now collect a few  facts about this function $f$.

\begin{lemma}\label{L:counterexample basic properties subgradient}
    Let $f$ be the function from \Cref{def:counterexample matrix kinky}. 
    \begin{enumerate}
        \item 
        $f(W) = f_{\diag}(\diag_2(W))$ for $f_{\diag} : \mathbb{R}^2 \to \mathbb{R}$ and
        \begin{equation*} 
            f_{\diag}(w_1,w_2) = c \vert w_1 + w_2 \vert + \vert w_1 - w_2 \vert.
        \end{equation*}
        \item $f_{\diag}$ and $f$ are convex, continuous and nonsmooth.
        \item $\inf f = \inf f_{\diag} = 0$.
        \item ${\rm{argmin}}~f_{\diag} = \{ 0 \}$ and $W^\star$ minimizes $f$ if and only if $W^\star_{1,1} = W^\star_{2,2} = 0$.
        \item The subdifferential of $f_{\diag{}}$ is
        \begin{eqnarray} \label{eq:subgrad-diag}
   \partial f_{\diag{}}(w)
   = c \Sign(w_{1}+w_{2}) \begin{bmatrix}
        1 \\1 
    \end{bmatrix} +\Sign(w_{1}-w_{2}) \begin{bmatrix}
        1 \\-1 
    \end{bmatrix}.
\end{eqnarray}
        
        \item The functions $f_{\diag} : (\mathbb{R}^2, \Vert \cdot \Vert_2) \to (\mathbb{R}, \vert \cdot \vert)$ and $f : (\mathbb{R}^{m \times n}, \Vert \cdot \Vert_F) \to (\mathbb{R}, \vert \cdot \vert)$ are $G$-Lipschitz continuous, with $G=\sqrt{2(1+ c^2)} \leq 2$.
    \end{enumerate}
\end{lemma}

\begin{proof}~~

    \begin{enumerate}
        \item Immediate.
        \item Trivial.
        \item $f(W)\geq 0$, and $f(0) = 0$, therefore $\inf_W f(W) = 0$. The same holds for $f_{\diag}$.
        \item Immediate.
        \item This follows from basic sum and composition rules for convex continuous functions and linear operators \cite[Proposition 3.28 \& Theorem 3.30]{Pey}, 
        and $\partial |x| = \Sign(x)$.
        \item The Lipschitz constant of the convex function $f_{\diag}$ is equal to $\sup_{w \in \mathbb{R}^2} \sup_{g \in \partial f_{\diag}(w)} \Vert g \Vert_2$.
        From the previous point, we know that 
        \begin{equation*}
            g \in \partial f_{\diag}(w) \implies
            g = (cs_1 + s_2, cs_1 - s_2),
        \end{equation*}
        for some $s_1 \in \Sign(w_{1}+w_{2})$ and $s_2 \in \Sign(w_{1}-w_{2})$.
        Then 
        \begin{equation*}
            \Vert g \Vert_2^2 = 
            2 c^2 s_1^2 + 2 s_2^2 \leq 2(c^2+1).
        \end{equation*} 
        It follows that $f_{\diag}$ is $\sqrt{2(c^2+1)}$-Lipschitz continuous, with $\sqrt{2(c^2+1)}\leq 2$ coming from the fact that $c \leq 1$.

        The Lipschitz constant of the convex function $f$ is equal to $\sup_{W \in \mathbb{R}^{m \times n}} \sup_{G \in \partial f(W)} \Vert G \Vert_F$.
        Using a chain rule as before leads to
        \begin{equation*}
            G \in \partial f(W) \implies
            G = \diag_2^*(g), 
        \end{equation*}
        where  $g \in \partial f_{\diag}(\diag_2(W))$.
        Because $\diag_2$ is a projection between Euclidean spaces, we have $\Vert \diag_2 \Vert_{\operatorname{op}} \leq 1$. 
        And because we just saw that $\Vert g \Vert \leq \sqrt{2(c^2+1)}$, we conclude that $\Vert G \Vert_F \leq \sqrt{2(c^2+1)}$ as well.
   
    \end{enumerate}
\end{proof}

Now we present an analog to the reduction Lemma \ref{lem:muon-reduction}, which shows that studying \eqref{eq:muon} on our counterexample is equivalent to studying the two-dimensional iterates of \eqref{eq:signum}.

\begin{restatable}[Muon reduction on our counterexample]{lemma}{MuonReductionPractical}
\label{L:muon reduction practical 2D}
Let $f$ be the function from \Cref{def:counterexample matrix kinky}, that we decompose as
\begin{equation*}
    f(W) = f_{\diag}(\diag_2(W)),
    \qquad
    f_{\diag}(w_1,w_2) = c \vert w_1 + w_2 \vert + \vert w_1 - w_2 \vert.
\end{equation*}

If $W_t \in  \R^{m\times n}$ are the iterates of~\eqref{eq:muon} applied to $f$, then the iterates $w_t = \diag_2(W_t) \in \mathbb{R}^2$ follow the~\eqref{eq:signum} algorithm applied to $f_{\diag}$, and we have $f(W_t) = f_{\diag}(w_t)$.

\end{restatable}

\begin{proof}

Let us first establish that every subgradient $G \in \partial f(W)$ is a diagonal matrix.
To see this, use the fact that $f = f_{\diag} \circ \diag_2$ where $f_{\diag}$ is convex continuous and $\diag_2 : \mathbb{R}^{m \times n} \to \mathbb{R}^2$ is a linear operator.
That allows to write \cite[Proposition 3.28]{Pey}
\begin{equation*}
    \partial f(W) = \diag_2^* (\partial f_{\diag}(\diag_2(W))).
\end{equation*}
It is a simple exercise to see that $\diag_2^*$ is exactly 
\begin{equation*}
    \diag_2^* : \mathbb{R}^2 \to \mathbb{R}^{m \times n}, \quad
    \diag_2^*(w_1, w_2) = \Diag(w_1,w_2, 0, \dots, 0).
\end{equation*}
In particular we see that every subgradient of $f$ is a diagonal matrix.

Let us now consider the iterates $W_t$ of~\eqref{eq:muon}.
At every iteration, the momentum update reads
\[
M_t = \beta M_{t-1} + (1-\beta) G_t,
\]
where $G_t \in \partial f(W_t)$ is a diagonal matrix.
We further know that the momentum matrix is initialized with $M_{-1} =0$ which is also diagonal.
By induction, we deduce that $M_t$ is diagonal for every $t \geq 0$.
We can then use~\Cref{lem:polar-diag} to establish that
\begin{equation*}
\polar(M_t) = \sign(M_t).
\end{equation*}
Consider now $w_t := \diag_2(W_t)$, and use the linearity of $\diag_2$ to write
\begin{eqnarray*}
    w_{t+1} &=& \diag_2(W_{t+1}) \\
    &=& \diag_2(W_t) - \lambda_t \diag_2(\polar(M_t)) \\
    &=&
    w_t - \lambda_t \diag_2(\sign (M_t)).
\end{eqnarray*}
Because  $\diag_2 \circ \sign = \sign \circ \diag_2$, we deduce that
\begin{equation*}
    w_{t+1} = w_t - \lambda_t \sign( \diag_2(M_t))
    =
    w_t - \lambda_t \sign(m_t)
\end{equation*}
after defining $m_t := \diag_2(M_t)$.
Moreover we have seen earlier that $G_t = \diag_2^*(g_t)$, where $g_t \in \partial f_{\diag}(\diag_2(W_t)) = \partial f_{\diag}(w_t)$.
This means that $\diag_2(G_t) = g_t$.
We can therefore write 
\begin{equation*}
    m_t = \diag_2(M_t)
    =
    \beta \diag_2(M_{t-1}) +(1- \beta) \diag_2(G_t)
    =
    \beta m_{t-1} + (1 - \beta) g_t.
\end{equation*}
We finally have verified that
\begin{eqnarray*}
    m_t &=& \beta m_{t-1} + (1 - \beta) g_t, \quad g_t \in \partial f_{\diag}(w_t), \\
    w_{t+1} &=& w_t - \lambda_t \sign(m_t),
\end{eqnarray*}
meaning that $w_t$ follows the~\eqref{eq:signum} algorithm applied to $f_{\diag}$.

\end{proof}

\subsection{Proofs of Counterexamples \ref{thm-counterex-dec-steps} and \ref{Cex:offline stepsizes momentum}}
\label{S:proofs counterexamples}

\maincounterexample*

\begin{proof}[Proof of Counterexample \ref{thm-counterex-dec-steps}]
First of all, we know thanks to~\Cref{L:muon reduction practical 2D} that we mostly need to examine the iterates of~\eqref{eq:signum} applied to the function $f_{\diag{}} : \mathbb{R}^{2} \to \mathbb{R}$,
\begin{equation*} 
    f_{\diag{}}(w) = c|w_1+w_2| + |w_1-w_2|.
\end{equation*}
The iterates of~\eqref{eq:signum} verify in this case
\[
g_t\in \partial f_{\diag{}}(w_t),
\qquad
m_t=\beta m_{t-1}+(1-\beta)g_t,
\qquad
w_{t+1}=w_t-\lambda_t \sign(m_t).
\]

Before starting our analysis of this algorithm, we need to focus on the following quantities, that we will later need to track along the iterations:
\[
\lambda:=\lim_{t\to\infty}\lambda_t,
\quad
R_t:=\frac{\lambda}{2}+S_t, \quad 
S_t :=\sum_{s=0}^\infty (-1)^s(\lambda_{t+s}-\lambda),
\quad t=0,1,\dots.
\]
Our assumptions that $\lambda_t$ is nonincreasing and converging to $\lambda$ imply, via the alternating-series test, that $S_t$ is a convergent series, and therefore is well defined.
We can further infer that $S_t \geq 0$, and even that $S_t >0$ in the case that $\lambda_t$ is strictly decreasing (which we assume when $\lambda=0$).
This implies that $R_t > 0$ holds true, whether $\lambda =0$ or not.
Another point to observe is that $S_t$ is, up to a sign change, the tail of the series $S_0$:
\begin{equation*}
    S_t = \sum_{k=t}^\infty (-1)^{k-t}(\lambda_k - \lambda)
    =
    (-1)^t \sum_{k=t}^\infty (-1)^{k}(\lambda_k - \lambda).
\end{equation*}
Because we know that $S_0$ is a converging series, the above means that $S_t$ tends to zero when $t \to + \infty$.
A last point to be made is that $R_t$ can be equivalently defined through the recursion $R_{t+1}=\lambda_t-R_t$.
Indeed
\begin{eqnarray*}
    R_{t+1} 
    &=& \frac{\lambda}{2} + S_{t+1} \\
    &=& \frac{\lambda}{2} + \sum_{s=0}^\infty (-1)^s(\lambda_{t+s +1}-\lambda) \\
    &=& \frac{\lambda}{2} - \sum_{s=1}^\infty (-1)^s(\lambda_{t+s}-\lambda) \\
    &=& \frac{\lambda}{2} + (\lambda_t - \lambda) - \sum_{s=0}^\infty (-1)^s(\lambda_{t+s}-\lambda) \\
    &=& \lambda_t - \frac{\lambda}{2} - S_t 
    = \lambda_t - R_t.
\end{eqnarray*}

Let us now start the analysis of the iterates $w_t$ of \eqref{eq:signum}.
Consider initializing the algorithm with $w_0  \in \mathcal{W}$, where the set of initializations $\mathcal{W}$ is defined by
\begin{equation*}
    \mathcal{W} := \left\{
r\begin{bmatrix}1\\ 1\end{bmatrix}
+
(R_0 + \delta)\begin{bmatrix}1\\[2pt]-1\end{bmatrix}  ~\bigg |~  r \geq 1, ~\vert \delta \vert < R_t, ~ t=0, 1, \ldots
\right\}.
\end{equation*}
Here a couple of observations need to be made, depending on the value of $\lambda$.
If $\lambda =0$ then $R_t> 0$ but also $R_t = S_t \to 0$. In this case $\delta$ must be equal to zero, and $\mathcal{W}$ is nonempty.
If $\lambda > 0$ then $R_t = \tfrac{\lambda}{2} + S_t \geq \tfrac{\lambda}{2}$.
In this case any choice $\vert \delta \vert < \tfrac{\lambda}{2}$ is valid, and $\mathcal{W}$ has a nonempty interior.

Consider now that we take $w_0 \in \mathcal{W}$ with parameters $r, \delta$.
We are going to prove by induction that, for every $t=0, 1, \ldots$
\begin{equation}\label{e:cexinduction}
    w_t=
r\begin{bmatrix}1\\ 1\end{bmatrix}
+
(\delta + (-1)^tR_t)\begin{bmatrix}1\\[2pt]-1\end{bmatrix}.
\end{equation}
It is clear that \eqref{e:cexinduction} is true at $t=0$, from our choice of $w_0$. 
Assume now that \eqref{e:cexinduction} holds for all $0, \ldots, t$, and let us show that it also holds for $t+1$.
To compute $\partial f_{\diag}(w_t)$, observe that 
\begin{equation*}
    \sign(w_{t,1} + w_{t,2}) = \sign(2 r) = +1,
    \qquad 
    \sign(w_{t,1} - w_{t,2})  = \sign(2(\delta + (-1)^t R_t)) = (-1)^t,
\end{equation*}
where the last equality stems from the fact that $\vert \delta \vert < R_t$.
Using \Cref{L:counterexample basic properties subgradient} we understand that $\partial f_{\diag}(w_t)$ is reduced to a unique subgradient, that we note $g_t$, and whose expression is given by
\[
g_t=
c\begin{bmatrix}1\\ 1\end{bmatrix}
+
(-1)^t\begin{bmatrix}1\\[2pt]-1\end{bmatrix}.
\]
We can now unroll the recursion defining the momentum sequence $m_t$, and write
\begin{eqnarray*}
m_t
&=& (1-\beta)\sum_{s=0}^t \beta^{t-s} g_s \\
&=& c(1-\beta)\sum_{s=0}^t \beta^{t-s}\begin{bmatrix}1\\ 1\end{bmatrix}
+
(1-\beta)\sum_{s=0}^t \beta^{t-s}(-1)^s\begin{bmatrix}1\\ -1\end{bmatrix} \\
&=& c(1-\beta)\sum_{s=0}^t \beta^s \begin{bmatrix}1\\ 1\end{bmatrix}
+
(1-\beta)(-1)^t\sum_{s=0}^t (-\beta)^s \begin{bmatrix}1\\ -1\end{bmatrix} \\
&=& c(1-\beta)\frac{1-\beta^{t+1}}{1-\beta}\begin{bmatrix}1\\ 1\end{bmatrix}
+
(1-\beta)(-1)^t\frac{1-(-\beta)^{t+1}}{1+\beta}\begin{bmatrix}1\\ -1\end{bmatrix} \\
&=& \underbrace{c\bigl(1-\beta^{t+1}\bigr)}_{=:a_t}\begin{bmatrix}1\\ 1\end{bmatrix}
+
(-1)^t\underbrace{\frac{1-\beta}{1+\beta}\bigl(1-(-\beta)^{t+1}\bigr)}_{=:b_t}\begin{bmatrix}1\\ -1\end{bmatrix}.
\end{eqnarray*}

Since $\beta \in [0,1)$ we know that $a_t \geq 0$.
We can also see that $b_t > a_t$ due to the fact that
\[
c = \frac{1 - \beta}{2}<\frac{1-\beta}{1+\beta},
\qquad 
1-\beta^{t+1}\le 1-(-\beta)^{t+1}.
\]
This means that the sign of $m_t$ will be governed by the coefficient of the
$\bigl[\begin{smallmatrix}1\\
-1\end{smallmatrix}\bigr]$-term, that is
\[
\sign(m_t)=(-1)^t\begin{bmatrix}1\\[2pt]-1\end{bmatrix}.
\]
Hence
\begin{eqnarray*} 
w_{t+1}
&=& w_t-\lambda_t\sign(m_t) \\
&=& r\begin{bmatrix}1\\ 1\end{bmatrix}
+
(\delta + (-1)^tR_t)\begin{bmatrix}1\\[2pt]-1\end{bmatrix}-\lambda_t(-1)^t\begin{bmatrix}1\\[2pt]-1\end{bmatrix} \\
&=&
r\begin{bmatrix}1\\ 1\end{bmatrix}
+
(\delta + (-1)^t(R_t-\lambda_t))\begin{bmatrix}1\\[2pt]-1\end{bmatrix} \\
&=&
r\begin{bmatrix}1\\ 1\end{bmatrix}
+
(\delta + (-1)^{t+1}R_{t+1})\begin{bmatrix}1\\[2pt]-1\end{bmatrix},
\end{eqnarray*} 
where the last step uses $R_{t+1}=\lambda_t-R_t$. This proves that \eqref{e:cexinduction} holds for $t+1$.

Now we simply write that
\[
f_{\diag{}}(w_t)-\inf f_{\diag{}}
=
c|w_{t,1}+w_{t,2}|+|w_{t,1}-w_{t,2}|
\geq c|w_{t,1}+w_{t,2}|
=
2cr
\ge 2c = 1 - \beta.
\]

To conclude the proof, it remains to formally use \Cref{L:muon reduction practical 2D} to convert this lower bound into one for \eqref{eq:muon}.
Consider the function $f(W) = f_{\diag} (\diag_2(W))$, and define the following set of initializations
\begin{equation*}
    \Omega = \{ W \in \mathbb{R}^{m \times n} \mid  \diag_2(W) \in \mathcal{W} \}. 
\end{equation*}
From what we know about $\mathcal{W}$, we can say that $\Omega$ is nonempty and has a nonempty interior when $\lambda>0$.
Consider now the iterates of \eqref{eq:muon} applied to $f$ and initialized with $W_0 \in \Omega$.
Defining $\hat w_t := \diag_2(W_t)$, we know by definition that $\hat w_0 \in \mathcal{W}$.
Using \Cref{L:muon reduction practical 2D}, we also know that $\hat w_t$ follow the \eqref{eq:signum} algorithm applied to $f_{\diag}$, with $f_{\diag}(\hat w_t) = f(W_t)$.
Combining what precedes, we finally obtain that
\begin{equation*}
    f(W_t) = f_{\diag}(\hat w_t) \geq 1 - \beta.
\end{equation*}
It remains to use the Lipschitzness of $f$ to convert this into a bound on the iterates themselves.
For every minimizer $W^\star$ of $f$, we can use the fact that $f$ is $2$-Lipschitz continuous (see \Cref{L:counterexample basic properties subgradient})  to write
\begin{equation*}
    1 - \beta \leq f(W_t) = \vert f(W_t) - f(W^\star) \vert \leq 2 \Vert W_t - W^\star \Vert_F.
\end{equation*}
To conclude the proof, we simply need to replace $\Omega$ with its interior when $\lambda > 0$ to obtain an open set of initializations.
\end{proof}

\MaincounterExampleBeta*
\begin{proof}[Proof of Counterexample \ref{Cex:offline stepsizes momentum}]
Let $W_t$ be the iterates of \eqref{eq:muon} applied to $f$.
Recall that we assumed the stepsize at each step to be chosen as a function of the previously encountered gradients. 
That is, there exist functions $\phi_t : (\mathbb{R}^{m \times n})^{t+1} \to \mathbb{R}$ defined offline such that $\lambda_t = \phi_t(G_0, G_1, \ldots, G_t)$. 
In the light of \Cref{L:muon reduction practical 2D}, let us define $f_{\diag}(w_1,w_2) = c \vert w_1 + w_2 \vert + \vert w_1 - w_2 \vert$ so that $f(W) = f_{\diag}(\diag_2(W))$.
Let us also define $w_t := \diag_2(W_t)$ which satisfy $f(W_t) = f_{\diag}(w_t)$ and follow the \eqref{eq:signum} algorithm
\begin{eqnarray*}
    m_t &=& \beta m_{t-1} + (1 - \beta) g_t, \quad g_t \in \partial f_{\diag}(w_t), \\
    w_{t+1} &=& w_t - \lambda_t \sign(m_t),
\end{eqnarray*}
where $G_t = \diag_2^*( g_t)$.
In particular we see that $\lambda_t$ can be written as an offline function $\hat \phi_t$ of the gradients $g_0, \dots, g_t$.

Our main task is now to provide a lower bound for this \eqref{eq:signum} algorithm.
A key goal is to show that for almost every initialization, the subgradient $g_t \in \partial f_{\diag}(w_t)$ can only take on
finitely many values. 
From the subdifferential formula \eqref{eq:subgrad-diag}, we define this finite set of values as
\[
\mathcal{G} = \left \{  c s_1 \begin{bmatrix}
        1 \\1 
    \end{bmatrix} +s_2 \begin{bmatrix}
        1 \\-1 
    \end{bmatrix}
    \ \Big | \ s_1, s_2 \in \{\pm 1 \} \right \}.
\]
Define now
\begin{equation} \label{eq:Elambda}
    E_\phi := \bigcup_{T \in \mathbb{N}_0} \left\{ 2\sum_{t=0}^{T-1} \hat \phi_t(h_0, \ldots, h_t) \delta_t \ \Big|\ \delta_t \in\{\pm1\}, ~ h_t\in \mathcal{G},
    ~ t=0, \ldots, T-1\right\},
\end{equation}
and
\begin{equation} \label{eq:Wlambda}
    \hat{\mathcal{W}} := \{w \in \mathbb{R}^2 \mid w_1+w_2 =0\}\cup\{w \in \mathbb{R}^2 \mid w_1 -w_2\in E_\phi\}.
\end{equation}
The set $E_\phi$ is countable: for each fixed $T$, the displayed set is finite,
and $E_\phi$ is a countable union of such finite sets. 
Hence $E_\phi$ has measure zero. Since $w\mapsto w_1 -w_2$ is a nonzero linear functional,
\[
\{w \mid w_1 -w_2\in E_\phi\}=\bigcup_{x\in E_\phi}\{w \mid w_1 -w_2=x\}
\]
is a countable union of affine hyperplanes, hence has measure zero. 
Taking the union with  $\{w \mid w_1 + w_2=0\}$ also gives a measure zero set, 
thus $\hat{\mathcal{W}}$ has measure zero.
We finally define 
\[
\mathcal{W} := \{ W \in \mathbb{R}^{m \times n} \mid \diag_2(W) \in \hat{\mathcal{W}} \}
\]
which also has measure zero.
Indeed,
\[
\mathcal{W}
=
\{W\mid W_{11}+W_{22}=0\}
\cup
\bigcup_{x\in E_\phi}\{W\mid W_{11}-W_{22}=x\},
\]
which is a countable union of affine hyperplanes in $\mathbb{R}^{m \times n}$.
We will impose from now on that $W_0 \notin \mathcal{W}$.
Note that this is equivalent to saying that $w_0 =\diag_2(W_0) \notin \hat{\mathcal{W}}$.

Let 
\begin{equation} \label{eq:p-q-t}
p_t:=w_{t,1} +w_{t,2} =   
\Big\langle w_t, \begin{bmatrix}
        1 \\1 
    \end{bmatrix} \Big\rangle, \qquad 
    q_t:=w_{t,1} -w_{t,2} =   
    \Big\langle w_t, \begin{bmatrix}
        1 \\-1 
    \end{bmatrix} \Big\rangle.
\end{equation}
Our initialization choice $w_0\notin \hat{\mathcal W}$ implies that  $p_0\neq 0$ and $q_0\notin E_\phi$. Since $0\in E_\phi$ (take $T=0$ in~\eqref{eq:Elambda}), we also have $q_0\neq 0$. 

We show by induction that, for all $t= 0, 1, \ldots$,
\begin{equation}\label{eq:offline-momentum-invariant}
p_t=p_0,
\qquad
q_t=q_0-2\sum_{s=0}^{t-1}\lambda_s \sign(q_s),
\qquad
q_t\neq 0.
\end{equation}
The claim is immediate at $t=0$.
Assume \eqref{eq:offline-momentum-invariant} holds at time $t$. By induction we have that $p_t \neq 0$ and $q_t \neq 0$, thus
\[
 \sign(p_0)\in\{\pm1\},
\qquad
 \sign(q_t)\in\{\pm1\}.
\]

Since $p_t\neq 0$ and $q_t\neq 0$, the subgradient~\eqref{eq:subgrad-diag} is unique, 
and is given by 
\[
g_t
=
c \sign(p_0) \begin{bmatrix}
        1 \\1 
    \end{bmatrix} + \sign(q_t) \begin{bmatrix}
        1 \\-1 
    \end{bmatrix}.
\]

Since $m_{-1}=0$ and $g_t$ is  a weighted sum of $ \begin{bmatrix}
        1 \\1 
    \end{bmatrix}$ and $ \begin{bmatrix}
        1 \\ -1 
    \end{bmatrix}$ for all $t$, 
we have
\begin{equation} \label{eq:mtrsofjserfzrf}
m_t=\alpha_t  \begin{bmatrix}
        1 \\1 
    \end{bmatrix}+\eta_t  \begin{bmatrix}
        1 \\ -1 
    \end{bmatrix}
\end{equation}
for some scalars $\alpha_t,\eta_t$. The momentum recursion gives
\[
\alpha_t=\beta\alpha_{t-1}+(1-\beta)c \sign(p_0),
\qquad
\eta_t=\beta\eta_{t-1}+(1-\beta)\sign(q_t),
\qquad
\alpha_{-1}=\eta_{-1}=0.
\]
Hence
\[
\alpha_t=c(1-\beta^{t+1}) \sign(p_0)
\qquad \implies \qquad 
|\alpha_t|\le c.
\]
Also, since $|\sign(q_j)|=1$ for all $j=0, \ldots, t$ and $\eta_{-1}=0$, we have
\[
| \eta_t| = |(1-\beta)\sum_{j=0}^t \beta^{ j}\sign(q_j) | \; \leq \; (1-\beta)\sum_{j=0}^t \beta^{ j} = 1- \beta^{t+1} \leq 1. 
\]
Consequently
\[
\sign(q_t)\eta_t
=
(1-\beta)+\beta \sign(q_t)\eta_{t-1}
\ge
(1-\beta)-\beta|\eta_{t-1}|
\ge
1-2\beta.
\]
Since $ 1-2\beta > c$, the above gives
\[
\sign(q_t)\eta_t>c\ge |\alpha_t| \geq 0.
\]
As a consequence of the above we have that
\begin{equation}\label{eq:sign-mom-beta12-invariants}
\sign(\eta_t)=\sign(q_t),
\qquad
|\eta_t|>|\alpha_t|.
\end{equation}

Consequently $m_t$ in~\eqref{eq:mtrsofjserfzrf} is given by
\[
m_t=(\alpha_t+\eta_t, \alpha_t-\eta_t).
\]
By \eqref{eq:sign-mom-beta12-invariants},
\[
\sign(\alpha_t+\eta_t)=\sign(q_t),
\qquad
\sign(\alpha_t-\eta_t)=-\sign(q_t).
\]
Hence
\[
\sign(m_t)
=
\begin{bmatrix}
        \sign(q_t) \\ -\sign(q_t) 
    \end{bmatrix}
=
\sign(q_t)  \begin{bmatrix}
        1 \\ -1 
    \end{bmatrix}.
\]

The update becomes
\begin{equation} \label{eq:w_tesfezfezfzer}
w_{t+1}=w_t-\lambda_t \sign(q_t)  \begin{bmatrix}
        1 \\ -1 
    \end{bmatrix}.
\end{equation}
Hence by definition of $p_t$ in~\eqref{eq:p-q-t} we have that
\[
   p_{t+1} = \left\langle w_{t+1}, \begin{bmatrix}
        1 \\ 1 
    \end{bmatrix} \right\rangle = 
    p_t-\lambda_t \sign(q_t)  \left\langle 
\begin{bmatrix}
1 \\ -1
\end{bmatrix},
\begin{bmatrix}
1 \\ 1
\end{bmatrix}
\right\rangle\;=\;p_t,
\]
which shows the first part of~\eqref{eq:offline-momentum-invariant} for $t+1$. For
$\delta_t=-\sign(q_t)\in\{\pm1\}$,  by definition of $q_t$ in~\eqref{eq:p-q-t}  we have
that 
\begin{eqnarray*}
q_{t+1} & = & \left\langle w_{t+1}, \begin{bmatrix}
        1 \\ -1 
    \end{bmatrix} \right\rangle \\
    & =& q_t-\lambda_t \sign(q_t)\left\langle 
\begin{bmatrix}
1 \\ -1
\end{bmatrix},
\begin{bmatrix}
1 \\ -1
\end{bmatrix}
\right\rangle \\
&=& q_t-2\lambda_t \sign(q_t),
\end{eqnarray*}
and unrolling over $t$ yields
\begin{equation*}
    q_{t+1} = q_0 -  2 \sum_{s=0}^{t} \lambda_s \sign(q_s),
\end{equation*}
which shows the second part of~\eqref{eq:offline-momentum-invariant} for $t+1$. For the
last part of~\eqref{eq:offline-momentum-invariant}, notice that
\begin{eqnarray*}
    q_{t+1} = 0 &\iff& q_0 =  2 \sum_{s=0}^{t} \lambda_s \sign(q_s) \\
    &\iff& q_0 =  2 \sum_{s=0}^{t} \hat \phi_s(g_0, \ldots, g_s) \sign(q_s).
\end{eqnarray*}
Since $g_s \in \mathcal{G}$ and $\sign(q_s)\in\{\pm1\}$ for all $s$, it follows that
$q_0 \in E_\phi$. By construction $q_0 \notin E_\phi$, hence $q_{t+1}\neq 0$.
This proves the claim for $t+1$. 

Therefore $p_t=p_0\neq 0$ for all $t$, and
\[
f_{\diag}(w_t)\ge c | w_{t,1} + w_{t,2}| = c|p_t|=c|p_0| = \frac{1-2 \beta}{2}\vert w_{0,1} + w_{0,2} \vert >0,
\]
which in turn means that
\begin{equation*}
    f(W_t) - \inf f \geq c|p_0|>0.
\end{equation*}
So we conclude that the iterates cannot converge to a minimizer.

For \eqref{eq:muon-reg}, set $\tilde{\lambda}_t = \lambda_t \|M_t\|_{\nuc}$. Since $M_t$ is determined by $G_0, \ldots, G_t$, $\tilde{\lambda}_t$ is covered by the same argument.
\end{proof}

\section{Proofs in \Cref{sec-muon-ef}}

\subsection{LMOs and compression operators}

\DualSubgradientCompression*
\begin{proof}
For the proof, let $\hat{\mathcal{C}}(W)=  \|W\|_*\LMO(W)$, where we have dropped the $\alpha^2$ re-scaling.
Expand the squares
\[
\|W-\alpha^2 \hat{\mathcal{C}}(W)\|_F^2
=
\|W\|_F^2+\alpha^4\|\hat{\mathcal{C}}(W)\|_F^2-2\alpha^2\langle W,\hat{\mathcal{C}}(W)\rangle.
\]
By definition of $\LMO$ and the dual norm,
\[
\langle W,\LMO(W)\rangle=\|W\|_*,
\]
hence
\[
\langle W,\hat{\mathcal{C}}(W)\rangle=\|W\|_*^2.
\]
Also
\[
\|\hat{\mathcal{C}}(W)\|_F=\|W\|_*\|\LMO(W)\|_F,
\]
and since $\|\LMO(W)\|\le 1$ and $\|U\|\ge \alpha\|U\|_F$ for every $U\in\mathbb{R}^{m\times n}$, we have
\[
\|\LMO(W)\|_F \le \frac{1}{\alpha}\|\LMO(W)\| \le \frac{1}{\alpha}.
\]
Therefore,
\[
\|\hat{\mathcal{C}}(W)\|_F^2 \le \frac{1}{\alpha^2}\|W\|_*^2.
\]
Substituting into the expansion gives
\[
\|W-\alpha^2 \hat{\mathcal{C}}(W)\|_F^2
\le
\|W\|_F^2+\left(\frac{\alpha^4}{\alpha^2}-2\alpha^2\right)\|W\|_*^2,
\]
which becomes
\[
\|W-\alpha^2 \hat{\mathcal{C}}(W)\|_F^2
\le
\|W\|_F^2-\alpha^2\|W\|_*^2.
\]
Finally, $\|W\|\le \beta\|W\|_F$ implies by duality that
\[
\|W\|_* \ge \frac{1}{\beta}\|W\|_F.
\]
Hence
\[
\|W-\alpha^2\hat{\mathcal{C}}(W)\|_F^2
\le
\|W\|_F^2-\frac{\alpha^2}{\beta^2}\|W\|_F^2
=
\left(1-\frac{\alpha^2}{\beta^2}\right)\|W\|_F^2.
\]
This proves that $\mathcal{C}(W) = \alpha^2\hat{\mathcal{C}}(W)$ is a
$
\delta=\frac{\alpha^2}{\beta^2}
$ compression operator.

To conclude the proof, we provide an explicit computation of the constants $\alpha,\beta$ when considering the operator norm.
Consider any $W \in \mathbb{R}^{m \times n}$ and let $r$ be its rank, with of course $r \leq \min\{m,n\}$.
Consider  $W = U \Sigma V^\top$ its reduced SVD where $\Sigma = \Diag(\sigma) \in \mathbb{R}^{r \times r}$, with $\sigma \in \mathbb{R}^r$.
Then
\begin{equation*}
    \Vert W \Vert_F = \Vert \sigma \Vert_2,
    \qquad  
    \Vert W \Vert_{\text{op}} = \Vert \sigma \Vert_\infty.
\end{equation*}
We can then use classic bounds in $\mathbb{R}^r$ to write
\begin{equation*}
    \frac{1}{\sqrt{r}} \Vert \sigma\Vert_2 \leq \Vert \sigma \Vert_\infty \leq \Vert \sigma \Vert_2,
\end{equation*}
and the conclusion follows from $\tfrac{1}{\sqrt{r}} \geq \tfrac{1}{\sqrt{\min\{m,n\}}}$.
\end{proof}

\subsection{Basic choices of LMO-based compressors and EF-Muon}
\label{S:LMO compressors basic examples}

As stated in \Cref{L:LMO gives compressor operator}, for any given norm $\Vert \cdot \Vert$, the corresponding LMO is a compression operator up to a multiplicative factor.
Below we list explicit choices of norms and their corresponding multiplicative factors, whose proof can be found in standard textbooks such as~\cite{golub13}. 

\begin{remark}[Basic LMO-based compressors]\label{R:LMO-compressors}
For any of the following norms $\Vert \cdot \Vert$, the operator
\begin{equation*}
    \mathcal{C}(W) = \alpha^2 \Vert W \Vert_* \LMO_{\|\cdot \|}(W)
\end{equation*}
is a $\delta$-compression operator in the sense of \Cref{L:LMO gives compressor operator}, where 
\begin{itemize}
    \item if $\Vert \cdot \Vert=\|\cdot\|_1$ on $\mathbb{R}^{d}$, 
    then $\Vert \cdot \Vert_* = \Vert \cdot \Vert_\infty$,  $\alpha=1$ and $\delta=1/d$.
    \item if $\Vert \cdot \Vert=\|\cdot\|_2$ on $\mathbb{R}^{d}$,
    then $\Vert \cdot \Vert_*=\Vert \cdot \Vert_2$, $\alpha=1$ and $\delta=1$.
    \item if $\Vert \cdot \Vert=\|\cdot\|_\infty$ on $\mathbb{R}^{d}$,
    then $\Vert \cdot \Vert_*=\Vert \cdot \Vert_1$, $\alpha=1/\sqrt d$ and $\delta=1/d$.
    \item if $\Vert \cdot \Vert=\|\cdot\|_p$ on $\mathbb{R}^{d}$,
    then $\Vert \cdot \Vert_*=\Vert \cdot \Vert_q$, $\alpha=d^{\min\left \{0, \tfrac{1}{p} - \tfrac{1}{2}\right \}}$, 
    and $\delta = d^{-2\lvert \frac1p-\frac12\rvert}$.
    \item if $\Vert \cdot \Vert=\|\cdot\|_{\operatorname{nuc}}$ on $\mathbb{R}^{m\times n}$,
    then $\Vert \cdot \Vert_*=\Vert \cdot \Vert_{\operatorname{op}}$, $\alpha=1$, 
    and $\delta=1/\min\{m,n\}$.
    \item if $\Vert \cdot \Vert=\|\cdot\|_{\operatorname{op}}$ on $\mathbb{R}^{m\times n}$,
    then $\Vert \cdot \Vert_*=\Vert \cdot \Vert_{\operatorname{nuc}}$,   
    and $\alpha^2 = \delta=1/\min\{m,n\}$.
\end{itemize}
These are worst-case bounds; for structured vectors (e.g., sparse or low rank gradients) the effective compression can be much better.
\end{remark}

Computing an element in the $\LMO_{\|\cdot \|}(W)$ can be done in closed form for the norms $\|\cdot\|_1$, $\|\cdot\|_2$, $\|\cdot\|_\infty$, $\|\cdot\|_{\operatorname{nuc}}$, see \cite[Section D.1]{gruntkowska2025errorfeedbackmuonfriends} for more details.
For the operator norm $\|\cdot\|_{\operatorname{op}}$, we prove here  that the polar factor belongs to the LMO.

\begin{lemma}\label{lem-lmo-op-norm-least-norm-polar}
For any $A\in \R^{m\times n}$, $\polar(A)$ is the unique least Frobenius norm element of 
$\LMO_{\|\cdot\|_{\op}}(A)$.
\end{lemma}
\begin{proof}
By duality of the spectral and nuclear norms~\cite{watson1992characterization},
\[
\LMO_{\|\cdot\|_{\op}}(A) = \partial \|A\|_{\nuc}.
\]

Let $r$ be the rank of $A$, and $A=U \Sigma V^\top$ be full SVD of $A$.
Write
\[
U=\begin{bmatrix}U_r & U_\perp\end{bmatrix},
\qquad
V=\begin{bmatrix}V_r & V_\perp\end{bmatrix},
\]
where $U_r\Sigma_r V_r^\top$ is the reduced SVD of $A$.

By \cite{watson1992characterization}, the subdifferential is equal to
\[
\partial \|A\|_{\operatorname{nuc}} = 
U_r V_r^\top  + \{ U_\perp E V_\perp^\top \mid E \in \R^{(m-r)\times (n-r)}, ~\|E\|_{\operatorname{op}}\leq 1 \}.
\]
Moreover, $\polar(A) = U_r V_r^\top$.

Consider any $E \in \R^{(m-r)\times (n-r)}$ with $\|E\|_{\operatorname{op}}\leq 1$, then
\[
U_r V_r^\top  + U_\perp E V_\perp^\top \in \partial \|A\|_{\operatorname{nuc}}.
\]
By orthogonal invariance of Frobenius norm,
\[
\|U_r V_r^\top  + U_\perp E V_\perp^\top\|^2_F
= \left \| U \begin{bmatrix}
    I_r & 0 \\ 
    0 & E
\end{bmatrix}V^\top \right\|^2_F = r + \|E\|^2_F \geq r = \|\polar(A)\|_F^2.
\]
Equality is attained if and only if $E=0$,
which concludes the result.
\end{proof}

We can now specialize \Cref{alg:efm} for the LMO-induced compression operator when considering the operator norm.

\begin{algorithm}
\caption{Error feedback Muon (EF-Muon)}
\label{algo:EF muon}
\begin{algorithmic}[1]
\Require stepsizes $\lambda_t>0$, momentum parameter $\beta\in[0,1)$, $r = \min\{m,n\}$
\State Initialize $W_0\in\R^{m\times n}$, $E_0=0$, and $M_{-1}=0$
\For{$t=0,1,\dots,T$}
    \State $G_t \in \R^{m\times n}$ \Comment{Sample a stochastic subgradient } 
    \State $
        M_{t}=\beta M_{t-1}+(1-\beta)G_t
    $ \Comment{ Update the momentum} 
    \State
    $
        P_t=E_t+\lambda_t M_{t}
    $  \Comment{Form the error-corrected message} 
    \State  
    $
        W_{t+1}=W_t-\tfrac{1}{r} \Vert P_t \Vert_{\operatorname{nuc}} \polar(P_t) 
    $ \Comment{Update the iterate with a spectral step}
    \State 
    $
        E_{t+1}=E_t + W_{t+1} - (W_t - \lambda_t M_t)
    $ \Comment{Update the error feedback memory}
\EndFor
\end{algorithmic}
\end{algorithm}

\subsection{LMO-based compression operator over layers}
\label{sec:Cartesian-compressor}
Neural networks are not parametrized by a single matrix, but instead by a Cartesian product of matrices
\begin{equation} \label{eq:cart-prod}
    W := (W^1, \ldots , W^L, \theta) 
\end{equation} 
where $W^{\ell} \in \R^{m_{\ell} \times n_{\ell}}$ is the weight matrix for all $\ell=1, \ldots, L$ layers, and $\theta \in \R^k$ contains all the non-matrix parameters of the entire network.

To define a subgradient method we need to now choose a norm over the product space~\eqref{eq:cart-prod}. To have the resulting subgradient method perform Muon-type updates over the weight matrices, we will choose a spectral norm over these parameters. It then remains to aggregate all parameters into one norm and use~\Cref{L:LMO gives compressor operator} to define a valid compression operator.

We now choose a norm over this product space so that 
\eqref{eq:muon-reg} with respect to this norm recovers a variant of MuonMax~\cite{Crawshaw2025,anthology}.
MuonMax was the first proposed variant of Muon~\cite{anthology}. 
We choose this variant because in a systematic study of Muon variants, 
it was found to be best performing regularized variant of Muon~\cite{Crawshaw2025}.
We define such norm in \Cref{proposition-muonmax-norm} along with corresponding compression
operator and LMO.
Note that up to a scalar factor, the compression applies the polar factor separately to
each layer.

\begin{lemma}[Product norm]\label{proposition-muonmax-norm}
Consider a Cartesian product of matrices as in \eqref{eq:cart-prod}, and for each $\ell =1, \ldots, L$ consider $d_\ell = \min\{m_\ell,n_\ell\}$.
If we define, for some $s>0$,
\begin{equation} \label{e-def-muonmax-norm}
    \|W\| := \left( \left( \max_{\ell \in [L]} \sqrt{\frac{d_{\ell}}{s}} \|W^\ell\|_{\text{op}} \right)^2 + k \|\theta\|_\infty^2 \right)^{1/2},
\end{equation}
then $\|\cdot\|$ is a norm on
\[
\mathbb{R}^{m_1\times n_1}\times\cdots\times\mathbb{R}^{m_L\times n_L}\times\mathbb{R}^k.
\]
Its dual norm is
\begin{equation*}
    \|W\|_* = \left( s y(W)^2 + \frac{\|\theta\|_1^2}{k} \right)^{1/2},
\end{equation*}
where
\begin{equation*}
    y(W) = \sum_{\ell=1}^L \frac{1}{\sqrt{d_\ell}} \|W^\ell\|_{\operatorname{nuc}}.
\end{equation*}
Moreover, for $W \neq 0$, the least product-Frobenius norm element of the set
$\LMO_{\|\cdot\|}(W)$ equals
\begin{equation*}
    \LMO^{\min}(W) = \frac{1}{\|W\|_*} \left( s \frac{y(W)}{\sqrt{d_1}} \operatorname{polar}(W^1), \ldots, s \frac{y(W)}{\sqrt{d_L}} \operatorname{polar}(W^L), \frac{\|\theta\|_1}{k}\sign(\theta) \right).
\end{equation*} 

In addition, the corresponding compressor is
\begin{equation}\label{e-muonmax-norm-compressor}
    \mathcal{C}(W) = \min \{s, 1/L \} \left( \frac{y(W)}{\sqrt{d_1}} \operatorname{polar}(W^1), \ldots, \frac{y(W)}{\sqrt{d_L}} \operatorname{polar}(W^L), \frac{\|\theta\|_1}{sk} \sign(\theta) \right).
\end{equation}

\end{lemma}

\begin{proof}
Let $v=(v_1, \ldots, v_{L+1}) \in \mathbb{R}^{L+1}$
with $v\neq 0$.
Define
\[
\|v\|_{\operatorname{hyb}}
=
\left(
\|v_{1:L}\|_\infty^2+v_{L+1}^2
\right)^{1/2}.
\]
This is the ``hybrid'' outer norm introduced in \cite{Crawshaw2025} that satisfies
\begin{eqnarray*}
    \|v\|_{\operatorname{hyb},*} &=& \left(
\|v_{1:L}\|_1^2+v_{L+1}^2 \right)^{1/2} \\
    \LMO_{\operatorname{hyb}}(v) &=& \frac{1}{\|v\|_{\operatorname{hyb},*}} \left( 
    \|v_{1:L}\|_1 \Sign(v_1), \ldots, 
    \|v_{1:L}\|_1 \Sign(v_L), v_{L+1} \right).
\end{eqnarray*} 

Now, set for all $\ell=1, \ldots, L$
\[
g_\ell(X):=\sqrt{\frac{d_\ell}{s}}\|X\|_{\op},
\qquad
g_{L+1}(z):=\sqrt{k}\|z\|_\infty,
\qquad
f(r):=\|r\|_{\operatorname{hyb}}.
\]
Then
\[
\|W\|
=
f\bigl(g_1(W^1),\ldots,g_L(W^L),g_{L+1}(\theta)\bigr),
\] 
so the claim follows from \cite[Lemma~3.3]{Crawshaw2025} after translating  the argmin convention to our argmax convention.

Indeed, since $f^* = \|\cdot\|_{\operatorname{hyb},*}$,
\[
g_{\ell,*}(X)=\frac{\|X\|_{\operatorname{nuc}}}{\sqrt{d_\ell/s}},
\qquad
g_{L+1,*}(z)=\frac{\|z\|_1}{\sqrt{k}},
\]
and therefore
\[
    \|W\|_* = \left( s y(W)^2 + \frac{\|\theta\|_1^2}{k} \right)^{1/2},
\]
where \[
    y(W) = \sum_{\ell=1}^L \frac{1}{\sqrt{d_\ell}} \|W^\ell\|_{\operatorname{nuc}}.
\]
Then the full LMO set is
\begin{eqnarray*} 
\LMO(W)
&=&
\left \{(\phi_1 U^1,\ldots,\phi_L U^L,\phi_{L+1}u^\theta)\mid \phi\in\Phi,\ u=(U^1,\ldots,U^L,u^\theta)\in\mathcal U \right \}\\
\Phi
&=&
\LMO_{\operatorname{hyb}}
\left(
    \frac{\|W^1\|_{\operatorname{nuc}}}{\sqrt{d_1/s}},
    \ldots,
    \frac{\|W^L\|_{\operatorname{nuc}}}{\sqrt{d_L/s}},
    \frac{\|\theta\|_1}{\sqrt{k}}
\right) \\
\mathcal{U} &=&
\left ( 
\LMO_{g_1}(W^1), \ldots, \LMO_{g_L}(W^L), \LMO_{g_{L+1}}(\theta)
\right ).
\end{eqnarray*}
Moreover,
\[
\LMO_{g_\ell}(X)=\sqrt{\frac{s}{d_\ell}}\partial\|X\|_{\nuc},
\qquad
\LMO_{g_{L+1}}(z)=\frac{1}{\sqrt{k}}\Sign(z),
\]
and using $\sqrt{d_\ell/s}>0$,
\[
\Phi = \frac{1}{\|W\|_*} \left (
\sqrt{s} y(W)
\Sign \left ({\|W^1\|_{\operatorname{nuc}}}\right), 
\ldots,
\sqrt{s} y(W)
\Sign \left ({\|W^L\|_{\operatorname{nuc}}}\right), 
\frac{\|\theta\|_1}{\sqrt{k}}
\right ).
\]
We now show that least product-Frobenius norm element of LMO is
\begin{equation} \label{eq:muonmax_lmo_min}
    \LMO^{\min}(W)= \frac{1}{\|W\|_*} \left( s \frac{y(W)}{\sqrt{d_1}} \operatorname{polar}(W^1), \ldots, s \frac{y(W)}{\sqrt{d_L}} \operatorname{polar}(W^L), \frac{\|\theta\|_1}{k}\sign(\theta) \right).
\end{equation}
The nonnegative block scalings
for blocks with $W_\ell \neq 0$ and for vector block are
\[
a_\ell = \frac{s y(W)}{\sqrt{d_\ell}\|W\|_*},
\qquad
b = \frac{\|\theta\|_1}{k\|W\|_*}.
\]
If $W^\ell=0$, the block $\ell$ does not affect $\|W\|_*$ and $y(W)$, 
and the least Frobenius norm choice is $0$.
Indeed, after the outer hybrid LMO fixes the nonnegative block scalings
the inner maximization separates over blocks.
For each matrix block, the least Frobenius norm maximizer of
\[
\max_{\|Z\|_{\op}\leq  a_\ell}\langle W^\ell, Z\rangle
\]
is $a_\ell\polar(W^\ell)$
by \Cref{lem-lmo-op-norm-least-norm-polar}.
For the vector block, the least Euclidean norm maximizer of
\[
\max_{\|z\|_\infty\leq b}\langle \theta,z\rangle
\]
is $b\sign(\theta)$. 
Therefore, \eqref{eq:muonmax_lmo_min} gives the least Frobenius norm element of
the LMO set.

Finally,
\begin{eqnarray*}
    \|W\|^2 &=& \left( \max_{\ell \in [L]} \sqrt{\frac{d_{\ell}}{s}} \|W^\ell\|_{\text{op}} \right)^2 + k \|\theta\|_\infty^2 \\
    &\geq& \left( \max_{\ell \in [L]} \frac{1}{\sqrt{s}} \|W^\ell\|_F \right)^2 + \|\theta\|_F^2 \\
    &=& \frac{1}{s} \left( \max_{\ell \in [L]} \|W^\ell\|_F \right)^2 + \|\theta\|_F^2 \\
    &\geq& \frac{1}{sL} \left( \sum_{\ell \in [L]} \|W^\ell\|_F^2 \right) + \|\theta\|_F^2 \\
    &\geq& \min \{1, 1/sL\} \left( \sum_{\ell \in [L]} \|W^\ell\|_F^2 + \|\theta\|_F^2 \right) \\
    &=& \min \{1, 1/sL\} \|W\|_F^2.
\end{eqnarray*}
Hence $\alpha = \min \{1, 1/\sqrt{sL} \}$ in \Cref{L:LMO gives compressor operator}, 
so
\[
    \mathcal{C}(W) = \min \{s, 1/L \} \left( \frac{y(W)}{\sqrt{d_1}} \operatorname{polar}(W^1), \ldots, \frac{y(W)}{\sqrt{d_L}} \operatorname{polar}(W^L), \frac{\|\theta\|_1}{sk} \sign(\theta) \right).
\]
\end{proof}

\subsection{MuonMax}\label{appdx-muonmax}
In this section we define MuonMax in \Cref{alg:reg-muonmax}, 
which is a variant of \eqref{eq:muon-reg} with respect to the 
product norm in \Cref{proposition-muonmax-norm}.

\begin{algorithm}[h!]
\caption{MuonMax}
\label{alg:reg-muonmax}
\begin{algorithmic}[1]
\Require stepsize schedule $\{\lambda_t\}_{t=0}^\infty$, momentum parameter $\beta\in[0,1)$, scaling parameter $s>0$
\State Initialize $W_0 \in \mathbb{R}^{m_1\times n_1}\times\cdots\times\mathbb{R}^{m_L\times n_L}\times\mathbb{R}^k$ 
\State Set $M_{-1}=0$
\For{$t=0,1,\dots,T$}
    \State $G_t=(G_t^1, \ldots, G_t^L, g_t^\theta)$ \Comment{Sample a stochastic subgradient } 
    \State $
        M_{t}=\beta M_{t-1}+(1-\beta)G_t
    $ \Comment{ Update the momentum}  
    \State $y(M_t) = \sum_{\ell=1}^L \frac{1}{\sqrt{d_\ell}} \|M_t^\ell\|_{\operatorname{nuc}}$
    \State 
    $
        \LMO(M_t) = \frac{1}{\|M_t\|_*}  \left(
         s\frac{y(M_t)}{\sqrt{d_1}}\operatorname{polar}(M_t^1),
        \ldots,
         s\frac{y(M_t)}{\sqrt{d_L}}\operatorname{polar}(M_t^L),
        \frac{\|m_t^\theta\|_1}{k}\sign(m_t^\theta)
    \right)
    $ \Comment{LMO}
    \State  
    $
        W_{t+1}=W_t-\lambda_t \|M_t\|_* \LMO(M_t)
    $ \Comment{Update the iterate} 
\EndFor
\State \Return $W_T$ 
\end{algorithmic}
\end{algorithm}

We also define EF-MuonMax in \Cref{alg:ef-muon-layers}.
This algorithm combines
error feedback with \Cref{alg:reg-muonmax}, by specializing \Cref{alg:efm} to 
compressor in \Cref{proposition-muonmax-norm}.
Thus by \Cref{thm:efm-convex} iterates of
\Cref{alg:ef-muon-layers} converge to a minimizer.

With a slight abuse of notation, we use the name MuonMax also for variants in 
which the signed momentum paired optimizer is replaced by another paired optimizer.

\begin{algorithm}[h!]
\caption{EF-MuonMax}
\label{alg:ef-muon-layers}
\begin{algorithmic}[1]
\Require stepsize $\lambda>0$, momentum parameter $\beta\in[0,1)$, scaling parameter $s>0$
\State Initialize $W_0 \in \mathbb{R}^{m_1\times n_1}\times\cdots\times\mathbb{R}^{m_L\times n_L}\times\mathbb{R}^k$ 
\State Set $E_0=0$, and set $M_{-1}=0$
\For{$t=0,1,\dots,T$}
    \State $G_t=(G_t^1, \ldots, G_t^L, g_t^\theta)$  \Comment{Sample a stochastic subgradient } 
    \State $
        M_{t}=\beta M_{t-1}+(1-\beta)G_t
    $ \Comment{ Update the momentum} 
    \State
    $
        P_t=E_t+\lambda M_{t}
    $  \Comment{Form the error-corrected message} 
    \State $y(P_t) = \sum_{\ell=1}^L \frac{1}{\sqrt{d_\ell}} \|P_t^\ell\|_{\operatorname{nuc}}$
    \State 
    $
        \mathcal{C}(P_t)=\min \{s, \frac{1}{L} \} \left(
         \frac{y(P_t)}{\sqrt{d_1}}\operatorname{polar}(P_t^1),
        \ldots,
         \frac{y(P_t)}{\sqrt{d_L}}\operatorname{polar}(P_t^L),
        \frac{\|p_t^\theta\|_1}{sk}\sign(p_t^\theta)
    \right)
    $ \Comment{Compressor}
    \State  
    $
        W_{t+1}=W_t-\mathcal{C}(P_t)
    $ \Comment{Update the iterate}
    \State 
    $
        E_{t+1}=P_t-\mathcal{C}(P_t)
    $ \Comment{Update the error feedback memory}
\EndFor
\State \Return $\displaystyle W_T$
\end{algorithmic}
\end{algorithm}

\subsection{Convergence bounds for error feedback with momentum}

The convergence result from the main body \Cref{thm:efm-convex} is a
direct corollary of \Cref{thm:efm-convex-general} below.

\begin{theorem}
\label{thm:efm-convex-general}
Consider \Cref{alg:efm}. Suppose that  $f:\R^{m\times n}\to\R$ is convex and admits a minimizer $W^\star$. At iteration $t$ let $G_t$ be a stochastic subgradient such that 
\[
    \mathbb{E}_t[G_t]\in \partial f(W_t),
    \qquad 
    \mathbb{E}_t\|G_t\|_F^2\leq \sigma^2.
\]
Let 
$\cC$ be a $\delta$-approximate compressor, meaning that for all $W\in\R^{m\times n}$,
\[
    \|\cC(W)-W\|_F^2\leq (1-\delta)\|W\|_F^2,
    \qquad \delta\in(0,1].
\]
Then, for the averaged iterate $\bar W_T=\frac{1}{T+1}\sum_{t=0}^T W_t$ and non-increasing stepsizes $\lambda_t$, we have
\[
    \mathbb{E} \left[ f(\bar{W}_T) \right] - f(W^\star) \leq \frac{\|W_0-W^\star\|_F^2}{2 \lambda_T (T+1)} + \sigma^2 \left( \frac{2 \sqrt{1-\delta}}{\delta} + \frac{\beta}{1-\beta} + \frac{1}{2} \right) \frac{\sum_{t=0}^T \lambda_t^2}{\lambda_T (T+1)}.
\]
\end{theorem}

\begin{proof}[Proof of \Cref{thm:efm-convex-general}]
Define
\[
    \widetilde W_t:=W_t-E_t,
    \qquad \Delta_t := \cC(P_t).
\]
Using the updates of \Cref{alg:efm}, we have
\begin{align*}
    \widetilde W_{t+1}
    &=
    W_{t+1}-E_{t+1}\\
    &=
    \bigl(W_t-\Delta_t\bigr)-\bigl(P_t-\Delta_t\bigr)\\
    &=
    W_t-P_t\\
    &=
    W_t-(E_t+\lambda_t M_t)\\
    &=
    \widetilde W_t-\lambda_t M_t.
\end{align*}
Now set
\[
    a:=\frac{\beta}{1-\beta},
    \qquad
    b_{T+1} := a \lambda_T,
    \qquad
    b_t := \beta (\lambda_t + b_{t+1})
    \qquad
    \eta_t := (1 - \beta) (\lambda_t + b_{t+1}),
\]
and
\[
    Y_t := \widetilde W_t- b_t M_{t-1}.
\]
Then
\begin{align*}
    Y_{t+1}
    &=
    \widetilde W_{t+1}-b_{t+1} M_t \\
    &=
    \widetilde W_t-(\lambda_t+b_{t+1})M_t \\
    &=
    \widetilde W_t-(\lambda_t+b_{t+1})
        \bigl(\beta M_{t-1}+(1-\beta)G_t\bigr) \\
    &=
    \widetilde W_t-b_tM_{t-1}-\eta_tG_t \\
    &=
    Y_t-\eta_tG_t.
\end{align*}
Next, unrolling the recursion for $b_t$, and using
$b_{T+1}=a\lambda_T$, gives
\[
    b_t
    =
    \sum_{j=t}^{T}\beta^{j-t+1}\lambda_j
    +
    \beta^{T-t+1}a\lambda_T.
\]
Since
\begin{align*}
    \sum_{j=t}^{T}\beta^{j-t+1}
    +
    a\beta^{T-t+1}
    =
    \sum_{k=1}^{T-t+1}\beta^k
    +
    a\beta^{T-t+1}
    =
    a(1 - \beta^{T-t+1}) + \beta^{T-t+1} a
    =
    a,
\end{align*}
and $\lambda_j \leq \lambda_t$ for all $j \geq t$, we have $b_t \leq a \lambda_t$.
Finally,
\[
    \eta_t
    =
    (1-\beta)(\lambda_t+b_{t+1}) 
    =
    (1-\beta)\sum_{j=t}^{T}\beta^{j-t}\lambda_j
    +
    \beta^{T-t+1}\lambda_T.
\]
The coefficients are nonnegative and sum to
\[
    (1-\beta)\sum_{j=t}^{T}\beta^{j-t}
    +
    \beta^{T-t+1}
    =
    1.
\]
Thus $\eta_t$ is a convex combination of
$\lambda_t,\lambda_{t+1},\dots,\lambda_T$. Since
$\lambda_t$ is nonincreasing,
$\lambda_T\le \eta_t\le \lambda_t$.

Let
\[
    S_t:=\mathbb{E}_t[G_t]\in\partial f(W_t).
\]
From the recursion $Y_{t+1}=Y_t- \eta_t G_t$, we obtain
\begin{align*}
    \mathbb{E}_t\|Y_{t+1}-W^\star\|_F^2
    &=
    \mathbb{E}_t\|Y_t-W^\star-\eta_t G_t\|_F^2\\
    &=
    \|Y_t-W^\star\|_F^2
    -2 \eta_t \langle \mathbb{E}_t[G_t], Y_t-W^\star\rangle
    + \eta_t^2 \mathbb{E}_t\|G_t\|_F^2\\
    &\leq
    \|Y_t-W^\star\|_F^2
    -2 \eta_t \langle S_t, Y_t-W^\star\rangle
    +\eta_t^2 \sigma^2.
\end{align*}
Since
\[
    Y_t=\widetilde W_t- b_t M_{t-1} = W_t-E_t- b_t M_{t-1},
\]
we have
\[
    \langle S_t, Y_t-W^\star\rangle
    =
    \langle S_t, W_t-W^\star\rangle
    -
    \langle S_t, E_t\rangle
    -
    b_t \langle S_t, M_{t-1}\rangle.
\]
Substituting this identity into the previous inequality and taking full expectation gives
\begin{align*}
    \mathbb{E} \|Y_{t+1}-W^\star\|_F^2 
    &\leq
    \mathbb{E} \|Y_t-W^\star\|_F^2
    -2\eta_t \mathbb{E} \langle S_t, W_t-W^\star\rangle 
    +2\eta_t \mathbb{E} \langle S_t, E_t\rangle \\
    &+2 b_t \eta_t \mathbb{E} \langle S_t, M_{t-1}\rangle
    +\eta_t^2\sigma^2.
\end{align*}
Rearranging and summing over $t = 0, \ldots, T$ gives
\begin{align}
    2 \sum_{t=0}^T \eta_t \mathbb{E} \langle S_t, W_t-W^\star\rangle \nonumber 
    &\leq
    \mathbb{E} \|Y_0-W^\star\|_F^2
    + 2 \sum_{t=0}^T \eta_t \mathbb{E} \langle S_t, E_t\rangle \\
    &+ 2 \sum_{t=0}^T b_t \eta_t \mathbb{E} \langle S_t, M_{t-1}\rangle
    + \sigma^2 \sum_{t=0}^T \eta_t^2.
    \label{eq:main-recursion-efm}
\end{align}
We can bound each of the terms on the RHS of \eqref{eq:main-recursion-efm} as follows.

First, we bound $\mathbb{E} \|M_t\|_F^2$. By unfolding the recursion
\[
    M_t=\beta M_{t-1}+(1-\beta)G_t,
\]
and using $M_{-1}=0$, we obtain
\[
    M_t
    =
    (1-\beta)\sum_{i=0}^t \beta^{t-i} G_i.
\]
The coefficients $(1-\beta)\beta^{t-i}$ are nonnegative and sum to
\[
    (1-\beta)\sum_{i=0}^t \beta^{t-i}
    =
    (1-\beta)\sum_{j=0}^t \beta^j
    =
    1-\beta^{t+1}
    \leq 1.
\]
Hence $M_t$ is a convex combination of $G_0,\dots,G_t$ and the origin. Since $V\mapsto \|V\|_F^2$ is convex,
\[
    \|M_t\|_F^2
    \leq
    (1-\beta)\sum_{i=0}^t \beta^{t-i}\|G_i\|_F^2.
\]
Taking expectations and using $\mathbb{E}\|G_i\|_F^2\leq \sigma^2$, we get
\[
    \mathbb{E}\|M_t\|_F^2
    \leq
    (1-\beta)\sum_{i=0}^t \beta^{t-i}\sigma^2
    \leq
    \sigma^2.
\]

Next, recall that
\[
    E_{t+1}=P_t-\cC(P_t),
    \qquad
    P_t=E_t+\lambda_t M_t.
\]
The compressor assumption implies
\[
    \|E_{t+1}\|_F^2
    =
    \|P_t-\cC(P_t)\|_F^2
    \leq
    (1-\delta)\|P_t\|_F^2
    =
    (1-\delta)\|E_t+\lambda_t M_t\|_F^2.
\]
If $\delta=1$, then 
\[
E_{t+1}=P_t-\cC(P_t)=0, 
\qquad t= 0, 1, \ldots.
\]
In this case, all terms involving $E_t$ vanish, and the remainder of the proof goes through unchanged. Thus it remains only to consider the case $0<\delta<1$.
Assume now that $0<\delta<1$. Using
\[
\|U+V\|_F^2\le (1+\alpha)\|U\|_F^2+(1+1/\alpha)\|V\|_F^2
\]
with
\[
\alpha=\frac{\delta}{2(1-\delta)},
\] 
we get
\[
    \|E_t+\lambda_t M_t\|_F^2
    \leq
    \frac{2-\delta}{2(1-\delta)}\|E_t\|_F^2
    +
    \frac{2 -\delta}{\delta}\lambda_t^2\|M_t\|_F^2.
\]
Multiplying by $(1-\delta)$ yields
\[
    \|E_{t+1}\|_F^2
    \leq
    \Bigl(1-\frac{\delta}{2}\Bigr)\|E_t\|_F^2
    +
    \frac{2(1-\delta)}{\delta}\lambda_t^2\|M_t\|_F^2.
\]
Taking expectations and using $\mathbb{E}\|M_t\|_F^2\leq \sigma^2$,
\[
    \mathbb{E}\|E_{t+1}\|_F^2
    \leq
    \Bigl(1-\frac{\delta}{2}\Bigr)\mathbb{E}\|E_t\|_F^2
    +
    \frac{2(1-\delta)}{\delta}\lambda_t^2\sigma^2.
\]
Since $E_0=0$, summing this inequality and multiplying by $2/\delta$ yields
\[
    \sum_{t=0}^T \mathbb{E}\|E_t\|_F^2
    \leq
    \frac{4(1-\delta)}{\delta^2} \sigma^2 \sum_{t=0}^T \lambda_t^2.
\]

By Cauchy--Schwarz and $\eta_t\le \lambda_t$,
\begin{align*}
    2 \sum_{t=0}^{T} \eta_t\mathbb E\langle S_t,E_t\rangle
    &\leq
    2\sum_{t=0}^{T}\eta_t
    \sqrt{\mathbb E\|S_t\|_F^2}
    \sqrt{\mathbb E\|E_t\|_F^2} \\
    &\le
    2\sigma\sum_{t=0}^{T}\eta_t \sqrt{\mathbb{E}\|E_t\|_F^2} \\
    &\le
    2\sigma
    \left(\sum_{t=0}^{T}\eta_t^2\right)^{1/2}
    \left(\sum_{t=0}^{T} \mathbb{E}\|E_t\|_F^2 \right)^{1/2} \\
    &\le
    2\sigma
    \left(\sum_{t=0}^{T}\lambda_t^2\right)^{1/2}
    \left(
        \frac{4(1-\delta)}{\delta^2}\sigma^2
        \sum_{t=0}^{T}\lambda_t^2
    \right)^{1/2} \\
    &=
    4\sigma^2\frac{\sqrt{1-\delta}}{\delta}
    \sum_{t=0}^{T}\lambda_t^2.
\end{align*}

Similarly,
\begin{align*}
    2 b_t \eta_t \mathbb{E}\langle S_t,M_{t-1}\rangle
    &\leq
    2 b_t \eta_t \mathbb{E}\bigl[\|S_t\|_F\|M_{t-1}\|_F\bigr]\\
    &\leq
    2 b_t \eta_t
    \sqrt{\mathbb{E}\|S_t\|_F^2}
    \sqrt{\mathbb{E}\|M_{t-1}\|_F^2}\\
    &\leq
    2 b_t \eta_t \sigma^2 \\
    &\leq
    2a \sigma^2 \lambda_t^2,
\end{align*}
where the last inequality uses $b_t \leq a \lambda_t$ and $\eta_t \leq \lambda_t$.

Applying the two bounds above to \eqref{eq:main-recursion-efm},
\begin{align*}
    2 \sum_{t=0}^T \eta_t \mathbb{E} \langle S_t, W_t-W^\star\rangle
    &\leq
    \mathbb{E} \|Y_0-W^\star\|_F^2
    + 4\sigma^2\frac{\sqrt{1-\delta}}{\delta} \sum_{t=0}^{T}\lambda_t^2
    + 2a \sigma^2 \sum_{t=0}^T \lambda_t^2
    + \sigma^2 \sum_{t=0}^T \eta_t^2 \\
    &\leq
    \mathbb{E} \|Y_0-W^\star\|_F^2
    + \sigma^2 \left( \frac{4 \sqrt{1-\delta}}{\delta} + 2a + 1 \right) \sum_{t=0}^T \lambda_t^2.
\end{align*}
Using $Y_0 = W_0$ and $f(W_t) - f(W^\star) \leq \langle S_t, W_t - W^\star \rangle$ from convexity of $f$, this means
\begin{equation*}
    \sum_{t=0}^T \eta_t \mathbb{E} \left[ f(W_t) - f(W^\star) \right] \leq \frac{\|W_0-W^\star\|_F^2}{2} + \sigma^2 \left( \frac{2 \sqrt{1-\delta}}{\delta} + a + \frac{1}{2} \right) \sum_{t=0}^T \lambda_t^2.
\end{equation*}
Since $\eta_t \geq \lambda_T$ and
$f(W_t)-f(W^\star)\geq 0$, 
dividing both sides by $\lambda_T (T+1)$ yields
\begin{equation*}
    \frac{1}{T+1} \sum_{t=0}^T \mathbb{E} \left[ f(W_t) - f(W^\star) \right] \leq \frac{\|W_0-W^\star\|_F^2}{2 \lambda_T (T+1)} + \sigma^2 \left( \frac{2 \sqrt{1-\delta}}{\delta} + a + \frac{1}{2} \right) \frac{\sum_{t=0}^T \lambda_t^2}{\lambda_T (T+1)},
\end{equation*}
so by convexity of $f$,
\begin{equation*}
    \mathbb{E} \left[ f(\bar{W}_T) \right] - f(W^\star) \leq \frac{\|W_0-W^\star\|_F^2}{2 \lambda_T (T+1)} + \sigma^2 \left( \frac{2 \sqrt{1-\delta}}{\delta} + a + \frac{1}{2} \right) \frac{\sum_{t=0}^T \lambda_t^2}{\lambda_T (T+1)},
\end{equation*}
which proves the claimed convergence rate.
\end{proof}

We can now prove \Cref{thm:efm-convex} by plugging in a stepsize choice to
\Cref{thm:efm-convex-general}.

\MainTheoremEFM*

\begin{proof}[Proof of \Cref{thm:efm-convex}]
Recall that \Cref{thm:efm-convex-general} implies for any non-increasing
stepsizes $\lambda_t$ that
\begin{equation*}
    \mathbb{E} \left[ f(\bar{W}_T) \right] - f(W^\star) \leq \frac{\|W_0-W^\star\|_F^2}{2 \lambda_T (T+1)} + \sigma^2 \left( \frac{2 \sqrt{1-\delta}}{\delta} + \frac{\beta}{1-\beta} + \frac{1}{2} \right) \frac{\sum_{t=0}^T \lambda_t^2}{\lambda_T (T+1)},
\end{equation*}
For the particular stepsize choice $\lambda_t = 1/\sqrt{t+1}$, we have
\[
    \sum_{t=0}^{T}\lambda_t^2=\sum_{t=1}^{T+1}\frac{1}{t}\leq 1 + \log (T+1),
\]
so
\begin{equation*}
    \mathbb{E} \left[ f(\bar{W}_T) \right] - f(W^\star) \leq \frac{\|W_0-W^\star\|_F^2}{2 \sqrt{T+1}} + \sigma^2 \left( \frac{2 \sqrt{1-\delta}}{\delta} + \frac{\beta}{1-\beta} + \frac{1}{2} \right) \frac{1 + \log(T+1)}{\sqrt{T+1}}.
\end{equation*}
\end{proof}

\paragraph{Background.}
Error feedback has been considered with Muon~\cite{gruntkowska2025errorfeedbackmuonfriends} but only in the smooth and distributed setting, and where the compression is done in addition to the matrix sign operation. That is, the matrix sign itself is not interpreted 
as a compression.

\section{Convergence, complexity and limitations that depend on $\beta$}
\label{asec:conv-complex-limit}
Our work shows that Muon does not converge on a certain class of functions. But there may still exist a \emph{complexity} result. Here we will clarify the difference, and properly define these notions. 

By convergence, what we mean is an any time convergence result~\cite{GarGow23}. In slightly informal terms, we say that a sequence of iterates $W_t$ (or some appropriate average of iterates) converges anytime if the suboptimality converges to zero, that is $\lim_{t\rightarrow \infty }f(W_t) -\inf f =0$. Our results in Counterexamples~\ref{Cex:offline stepsizes momentum} and~\ref{thm-counterex-dec-steps}  show that this is not possible, because for fixed $\beta$ the suboptimality remains bounded away from zero for \emph{all time}.

But sometimes it is enough to show that the suboptimality can be small, but not necessarily converge to zero. This is where complexity results matter. For instance,  a typical complexity would be as follows: given $T$ number of iterations, by setting the hyperparameters depending on $T$, show that $f(W_T) -\inf f \leq \frac{1}{T}$. Our counterexamples do not rule out such a result because given $T$ we could choose  $\beta = 1-\frac{1}{2T}$ and our lower bound~\eqref{eq:nonconv} from Counterexample~\ref{thm-counterex-dec-steps}  states that \[f(W_T)-\inf f \geq \frac{1}{2T}, \]
which does not rule out the possibility that $f(W_T)-\inf f \leq \frac{1}{T}$.  
As an illustration of such complexity result, in the stochastic smooth   setting, $O(T^{-1/4})$ stationarity complexity results have been established for Muon  by choosing $\beta = 1 - O(1/\sqrt{T})$~\cite[Theorem 4.1]{shen2025convergenceanalysisofmuon}.
But counter to this point, we know that when $\beta =1$, Muon~\eqref{eq:muon} 
will halt since the momentum buffer will become stationary $M_t = M_{-1}$, thus also ruling out convergence.

This brings us to another one of our stated limitations. All counterexamples rely on a fixed momentum  $\beta \in [0,1)$ coefficient. As such, there could exist a convergence result where, for instance $\beta_t = 1-\frac{1}{t}$. Our attempts at establishing such a proof  with such a scheduled momentum coefficient  have failed, so the question remains open whether convergence can be achieved with such a schedule.

\section{Experiments on counterexample function}
In this section we run Muon algorithm \eqref{eq:muon}, 
and \Cref{alg:efm}
with the function from \Cref{def:counterexample matrix kinky}, 
see \Cref{fig-counterex-level-set}.

\begin{figure}[h!]
\begin{center}
    \centering
        \includegraphics[width=0.7\textwidth]{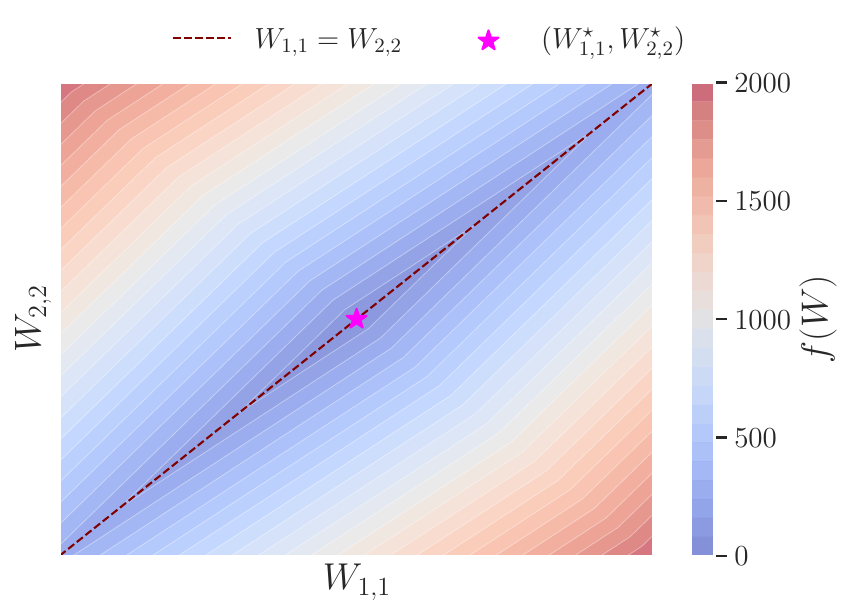}
    \caption{Level sets of $f$ from \Cref{def:counterexample matrix kinky}.} 
    \label{fig-counterex-level-set}
\end{center}
\end{figure}

\paragraph{Problem instance.}
Consider minimizing function from \Cref{def:counterexample matrix kinky}
with 
\[
c = \frac{1-\beta}{2(1+\beta)}, \qquad \beta = 0.9.
\] 
We choose the initial point satisfying
\[
(W_0)_{1,1} = 1 + \log 2, \qquad (W_0)_{2,2} = 1 - \log 2,
\] 
so that the conditions in Counterexample~\ref{thm-counterex-dec-steps}
hold. 
We set learning rate in \eqref{eq:muon} as $\lambda_t=\frac{1}{t+1}$,
and $\lambda_t=\frac{1}{\sqrt{t+1}}$ in \Cref{alg:efm}
for all $t=0, \ldots, T-1$.
We run the methods for $T=5000$ iterations.

\paragraph{Muon algorithm.}
Counterexample~\ref{thm-counterex-dec-steps} shows that, for every
$\beta \in [0,1)$, there exists an initialization $W_0$ for which
Muon algorithm fails to converge whenever
\[
c \in \left(0,\frac{1-\beta}{1+\beta}\right).
\]
\Cref{fig-counterex-muon-loss-w} illustrates this behavior for the
choices above. The iterates satisfy the invariant
\[
(W_t)_{1,1} + (W_t)_{2,2} = 2,
\]
so they cycle along a line in the
$((W_t)_{1,1},(W_t)_{2,2})$ plane. The resulting trajectory is shown in
\Cref{fig-counterex-muon-trajectory}.

\begin{figure}[h!]
\begin{center}
    \centering
        \includegraphics[width=\textwidth]{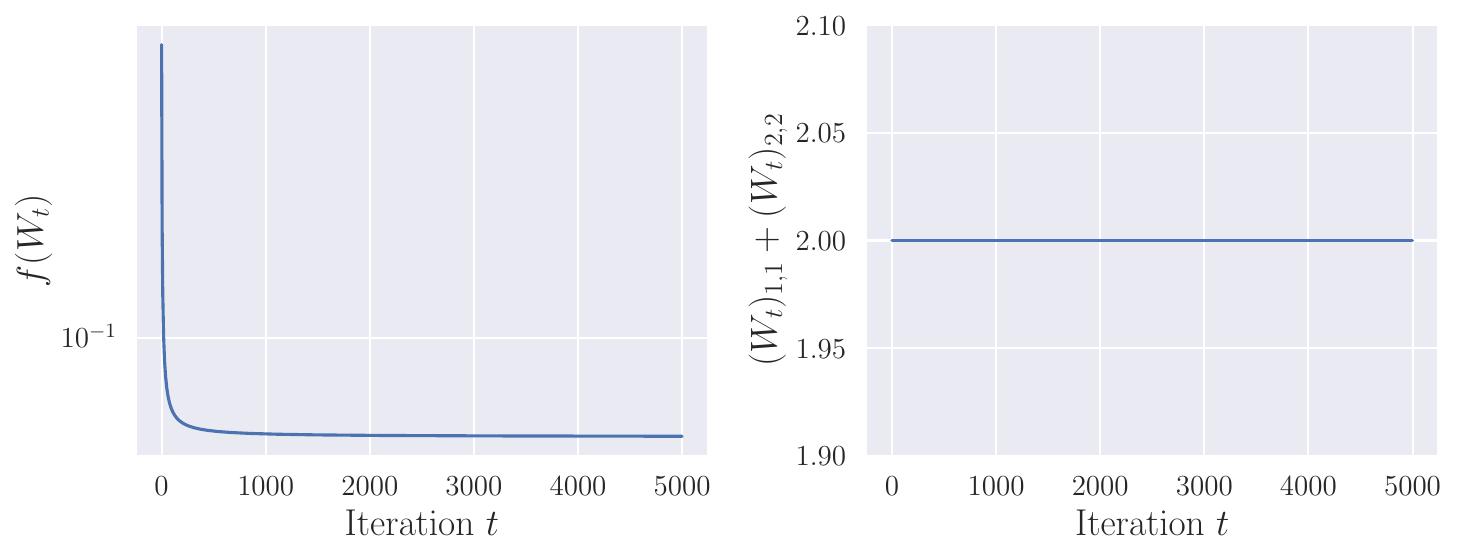}
    \caption{Muon algorithm \eqref{eq:muon} 
    on the counterexample:
(left) objective value $f(W_t)$; (right) invariant $(W_t)_{1,1}+(W_t)_{2,2}$.}
    \label{fig-counterex-muon-loss-w}
\end{center}
\end{figure}

\paragraph{EF-M algorithm.}
Consider \Cref{alg:efm} with learning rate
$\lambda_t = \frac{1}{\sqrt{t+1}}$.
In contrast to Muon method,
\Cref{thm:efm-convex} guarantees that \Cref{alg:efm}  converges for all 
convex Lipschitz functions.
\Cref{fig-counterex-ef-m-loss-w} illustrates 
that $f(W_t)\to f^\star$, and the right of 
\Cref{fig-counterex-muon-trajectory}
shows that the iterates converge, $((W_t)_{1,1}, (W_t)_{2,2}) \to (0, 0)$.

\begin{figure}[h!]
\begin{center}
    \centering
        \includegraphics[width=\textwidth]{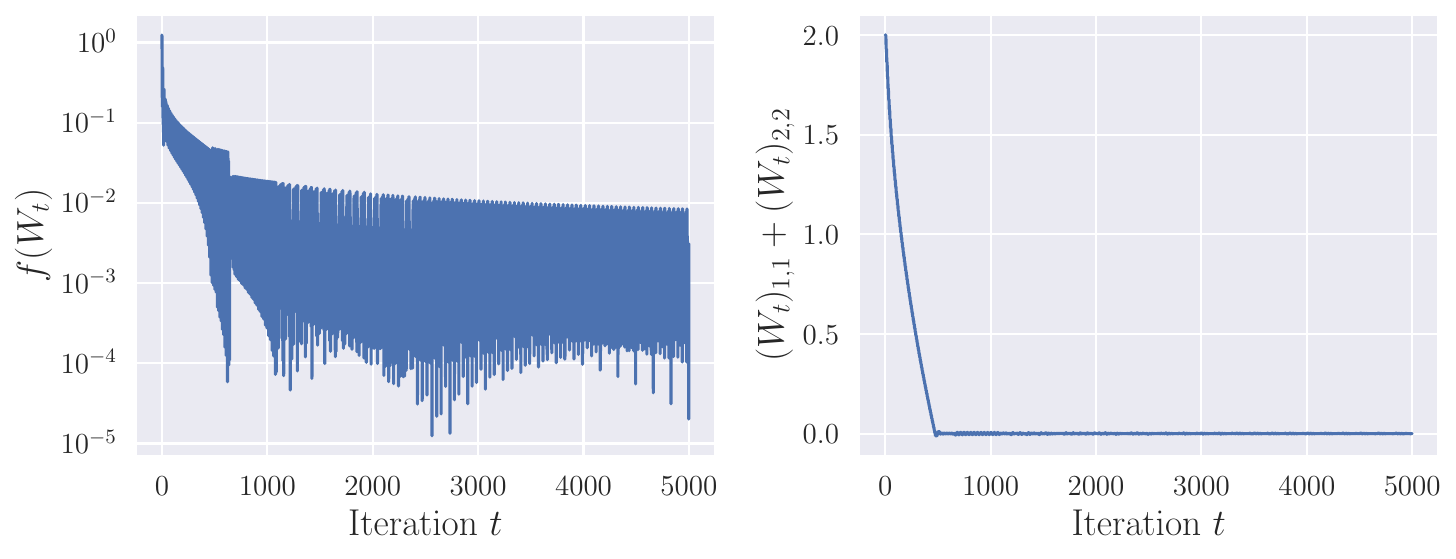}
    \caption{\Cref{alg:efm} on the counterexample:
(left) objective value $f(W_t)$; (right) invariant $(W_t)_{1,1}+(W_t)_{2,2}$.}
    \label{fig-counterex-ef-m-loss-w}
\end{center}
\end{figure}

\paragraph{Error feedback with spectral descent.}
We consider the case without momentum ($\beta=0$).
Similarly to the above,
\Cref{thm:efm-convex} 
guarantees that \Cref{alg:efm}  converges for all 
convex Lipschitz functions.
Indeed, \Cref{fig-counterex-polar-ef-gamma-t-loss-w} shows that $f(W_t)\to f^\star=0$,
and resulting trajectory in \Cref{fig-counterex-polar-ef-gamma-t-trajectory}
shows that $((W_t)_{1,1},(W_t)_{2,2}) \to (0, 0)=((W^\star)_{1,1},(W^\star)_{2,2})$.

\begin{figure}[h!]
\begin{center}
    \centering
        \includegraphics[width=\textwidth]{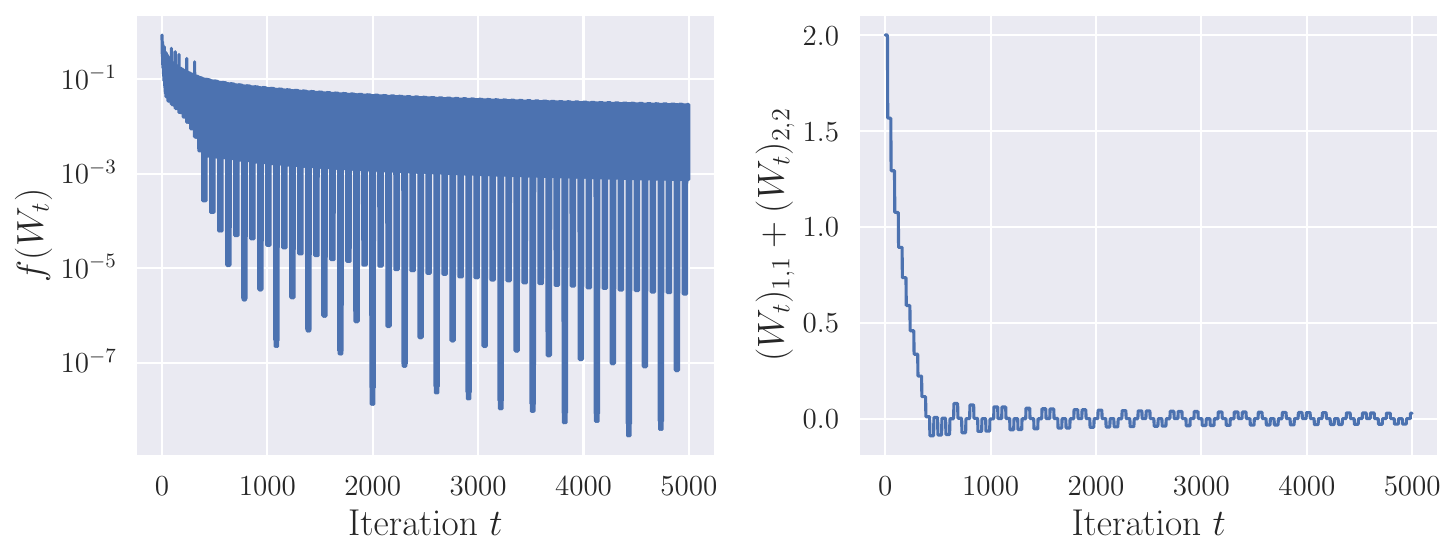}
    \caption{\Cref{alg:efm} without momentum ($\beta=0$) with $\lambda_t=\frac{1}{\sqrt{t+1}}$ on the counterexample:
(left) objective value $f(W_t)$; (right) $(W_t)_{1,1}+(W_t)_{2,2}$.}
    \label{fig-counterex-polar-ef-gamma-t-loss-w}
\end{center}
\end{figure}

\begin{figure}[h!]
\begin{center}
    \centering
        \includegraphics[width=0.7\textwidth]{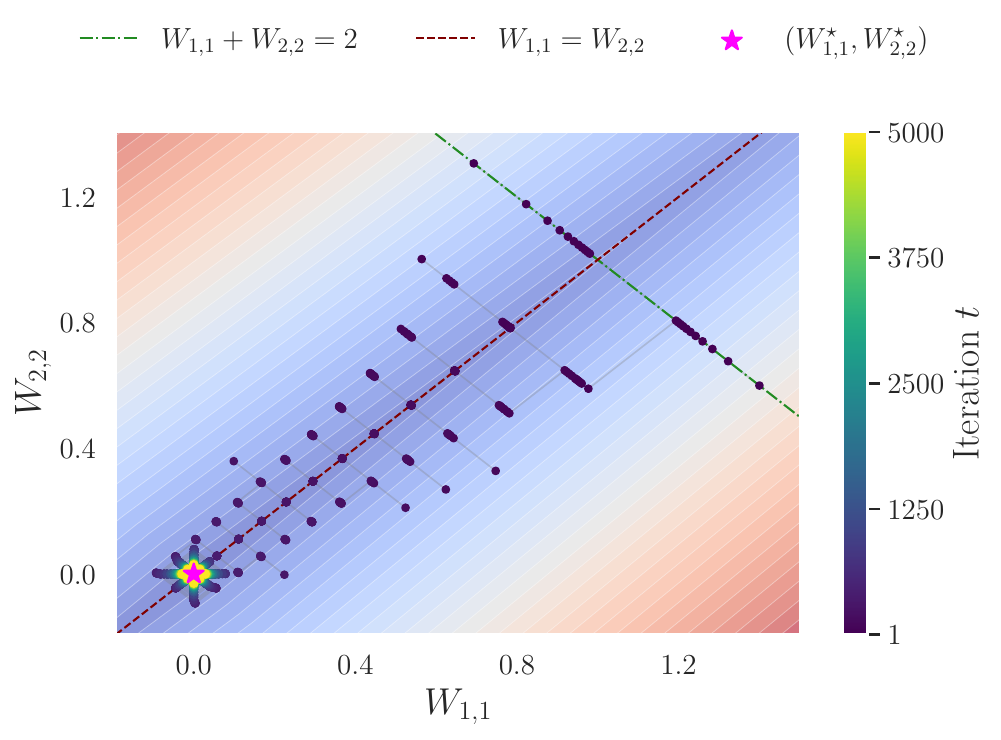}
    \caption{\Cref{alg:efm} without momentum ($\beta=0$) with $\lambda_t=\frac{1}{\sqrt{t+1}}$ trajectory in the $((W_t)_{1,1},(W_t)_{2,2})$ plane
    on the counterexample. Contours are level sets of $f$; dot color indicates the iteration $t$.} 
    \label{fig-counterex-polar-ef-gamma-t-trajectory}
\end{center}
\end{figure}

\section{Numerical experiments} \label{sec-numerical-experiments}

We compare performance of MuonMax \Cref{alg:reg-muonmax} and EF-MuonMax \Cref{alg:ef-muon-layers}
on image classification and language modeling tasks.
This section complements \Cref{sec-main-experiments}.

\paragraph{Parameter partition.}
We split all parameters into two groups: matrix parameters and paired parameters. Matrix parameters are weight matrices, excluding embeddings. These are updated by
vanilla Muon, MuonMax or EF-MuonMax. Paired parameters are
embeddings, biases, and the remaining one-dimensional parameters.
We use two update rules for the paired parameters.

\subsection{Image classification}
\label{sec:image-class}

\paragraph{Signed momentum paired optimizer.}
Matrix parameters are updated as above. For MuonMax and EF-MuonMax, paired
parameters are updated in the coupled paired mode, using signed momentum for
non-matrix parameters. We use momentum parameter $\beta=0.95$ and fix the scale
parameter at $s=1$.

\paragraph{Training setup.}
The model is a WideResNet-28-10 with 4 blocks per group and channel multiplier
10. It has 36M parameters. 
We train on CIFAR-10 dataset with 45K training images
and 5K validation images, batch size $128$, for 40 epochs
(around 14K steps). The learning rate schedule is cosine decay without
warmup. We use no weight decay.

\paragraph{Sweep.}
For each method, we sweep learning rate
\[
    \lambda_{\max} \in
    \{10^{-4},\,5\times 10^{-4},\,10^{-3},\,5\times 10^{-3},\,
      10^{-2},\,5\times 10^{-2}\}
\]
and Muon matrix learning rate (LR) multiplier
\[
    \{0.1,\,1,\,10,\,30,\,100\}.
\]
We also sweep Nesterov momentum over
$\{\mathrm{True},\mathrm{False}\}$.
This gives 60 trials per method. We select the best run by final validation
cross-entropy loss.
Results are shown in \Cref{tab:cifar10-results}.

\begin{table}[h]
\centering
\small
\begin{tabular}{lcccc}
\toprule
Method & $\lambda_{\max}$ & LR multiplier & Nesterov & Validation CE \\
\midrule
Vanilla Muon & $10^{-4}$ & $1$ & True & $0.287$ \\
MuonMax & $5\!\times\!10^{-4}$ & $100$ & False & \textbf{0.265} \\
EF-MuonMax & $5\!\times\!10^{-3}$ & $30$ & True & $0.388$ \\
\bottomrule
\end{tabular}
\vspace{0.5em}
\caption{Best configurations for image classification task.}
\label{tab:cifar10-results}
\end{table}

\paragraph{Vanilla Muon.}
We also include a vanilla Muon method with SGD paired optimizer. This
baseline is tuned on the same sweep as other methods.

\subsection{Language modeling}
\label{sec-lang-modeling}

\paragraph{NAdamW paired optimizer.}
Matrix parameters are updated as above. Paired parameters are updated by a
separate NAdamW optimizer~\cite{dozat2016incorporating,medapati_trai_neur_netw25},
with
\[
    \beta_1 = 0.9, \qquad
    \beta_2 = 0.999, \qquad
    \epsilon = 10^{-8}.
\]
This is not covered by the theory, but it is common in practice.

\paragraph{Training setup.}
The model is a $12$-layer RoPE transformer with embedding dimension $1024$, MLP
dimension $4096$, $16$ attention heads, RMSNorm, and QK-norm. It has 206M
parameters. We train on FineWeb-Edu 10B with sequence length $1024$, batch size
$64$, for 100K steps.
The learning rate schedule is cosine decay with $10\%$ warmup. We use momentum
parameter $\beta=0.95$ and weight decay $10^{-3}$. The scale parameter is fixed
at $s=1$.

\paragraph{Sweep.}
For each method, we sweep learning rate
\[
    \lambda_{\max} \in
    \{10^{-4},\,5\times 10^{-4},\,10^{-3},\,5\times 10^{-3},\,
      10^{-2},\,5\times 10^{-2}\}
\]
and Muon matrix LR multiplier
\[
    \{0.1,\,1,\,10,\,30,\,100\}.
\]
This gives 30 trials per method. We select the
best run by final validation cross-entropy. 

\paragraph{Vanilla Muon.}
We also include a vanilla Muon baseline with NAdamW paired optimizer. This
baseline is independently tuned, with
$\beta = 0.956$, LR multiplier equal to $28.6$ and weight decay of $0.024$,
with NAdamW paired hyperparameters
$\beta_1 = 0.966$,
$\beta_2 = 0.863$, 
and weight decay of $0.013$. The hyperparameter tuning is inherited from \texttt{init2winit} and followed the methodology outlined in~\cite{medapati_trai_neur_netw25}.
In particular, Muon was tuned on all the workloads of the \texttt{init2winit} codebase separately, with 400 runs per sweep. For each sweep, the hyperparameters were sampled randomly from the ranges given in \Cref{tab:muon-sweep-hparams};
final configuration is the one minimizing the cost function defined
in~\cite{medapati_trai_neur_netw25}.

\begin{table}[h]
\centering
\begin{tabular}{lll}
\toprule
Hyperparameter & Range/Values & Scale \\
\midrule
Base LR & $[0.00001, 0.005]$ & Log \\
Paired weight decay & $[0.01, 0.5]$ & Log \\
Muon weight decay & $[0.001, 0.5]$ & Log \\
Paired $\beta_1$ & $[0.5, 0.999]$ & Log \\
Paired $\beta_2$ & $[0.5, 0.999]$ & Log \\
Muon $\beta$ & $[0.2, 0.999]$ & Linear \\
Muon matrix LR Multiplier & $[0.01, 10000]$ & Log \\
Muon Nesterov & $\{ \texttt{False}, \texttt{True} \}$ & --- \\
Label smoothing & $\{0.0, 0.2\}$ & --- \\
Warmup fraction & $\{0.02, 0.05, 0.1\}$ & --- \\
\bottomrule
\end{tabular}
\caption{Vanilla Muon hyperparameter search space configuration for language modeling task.}
\label{tab:muon-sweep-hparams}
\end{table}

\end{document}